%% file: main.tex
\documentclass{article}

\usepackage{microtype}
\usepackage{graphicx}
\usepackage{subfigure}
\usepackage{booktabs} % for professional tables
\usepackage{hyperref}

\input{header}

\usepackage[accepted]{styles/arxiv/icml2022} % arxiv

\icmltitlerunning{Learning Infinite-horizon Average-reward Markov Decision Process with Constraints}

\begin{document}

\twocolumn[
\icmltitle{Learning Infinite-Horizon Average-Reward Markov Decision Processes \\ with Constraints}

% It is OKAY to include author information, even for blind
% submissions: the style file will automatically remove it for you
% unless you've provided the [accepted] option to the icml2021
% package.

% List of affiliations: The first argument should be a (short)
% identifier you will use later to specify author affiliations
% Academic affiliations should list Department, University, City, Region, Country
% Industry affiliations should list Company, City, Region, Country

% You can specify symbols, otherwise they are numbered in order.
% Ideally, you should not use this facility. Affiliations will be numbered
% in order of appearance and this is the preferred way.
\icmlsetsymbol{equal}{*}

\begin{icmlauthorlist}
\icmlauthor{Liyu Chen}{usc}
\icmlauthor{Rahul Jain}{usc}
\icmlauthor{Haipeng Luo}{usc}
\end{icmlauthorlist}

\icmlaffiliation{usc}{University of Southern California}

\icmlcorrespondingauthor{Liyu Chen}{liyuc@usc.edu}

% You may provide any keywords that you
% find helpful for describing your paper; these are used to populate
% the "keywords" metadata in the PDF but will not be shown in the document
\icmlkeywords{Machine Learning, ICML}

\vskip 0.3in
]

% this must go after the closing bracket ] following \twocolumn[ ...

% This command actually creates the footnote in the first column
% listing the affiliations and the copyright notice.
% The command takes one argument, which is text to display at the start of the footnote.
% The \icmlEqualContribution command is standard text for equal contribution.
% Remove it (just {}) if you do not need this facility.

\printAffiliationsAndNotice{}  % leave blank if no need to mention equal contribution
%\printAffiliationsAndNotice{\icmlEqualContribution} % otherwise use the standard text.

\begin{abstract}
We study regret minimization for infinite-horizon average-reward Markov Decision Processes (MDPs) under cost constraints.
We start by designing a policy optimization algorithm with carefully designed action-value estimator and bonus term,
and show that for ergodic MDPs, our algorithm ensures $\tilo{\sqrt{T}}$ regret and constant constraint violation, where $T$ is the total number of time steps. %(omitting other dependency).
This strictly improves over the algorithm of~\citep{singh2020learning}, whose regret and constraint violation are both $\tilo{T^{2/3}}$.
%The key components of our method are policy optimization, a carefully designed action-value estimator, and a special bonus term.
Next, we consider the most general class of weakly communicating MDPs. Through a finite-horizon approximation, we develop another algorithm with $\tilo{T^{2/3}}$ regret and constraint violation, which can be further improved to $\tilo{\sqrt{T}}$ via a simple modification,
albeit making the algorithm computationally inefficient.
As far as we know, these are the first set of provable algorithms for weakly communicating MDPs with cost constraints.
\end{abstract}

\section{Introduction}
\input{intro}

\section{Preliminaries}
\input{pre}

\section{Results for Ergodic MDPs}
\label{sec:ergodic}
\input{ergodic}

\section{Results for Weakly Communicating MDPs}
\label{sec:weak}
\input{weak}

\section*{Acknowledgements}
LC thanks Chen-Yu Wei for helpful discussions.

%\section{Conclusion}
%In this work, we study regret minimization in infinite-horizon average-reward CMDPs and propose algorithms under both the ergodic and the weakly communicating assumptions.
%There are many important future directions, such as improving the upper bounds under both assumptions, and establishing lower bounds.

% In the unusual situation where you want a paper to appear in the
% references without citing it in the main text, use \nocite
%\nocite{langley00}

\bibliography{ref}
\bibliographystyle{styles/icml/icml2022}

\newpage
\onecolumn
\appendix

\section{Preliminaries for the Appendix}
\label{app:pre}
\input{app-pre}

\section{Omitted Details for \pref{sec:ergodic}}
\input{app-ergodic}

\section{Omitted Details for \pref{sec:weak}}
\input{app-weak}

\section{Concentration Inequalities}
\input{app-aux}

\end{document}

%% file: header.tex
\usepackage[utf8]{inputenc} % allow utf-8 input
\usepackage[T1]{fontenc}    % use 8-bit T1 fonts
\usepackage{hyperref}       % hyperlinks
\usepackage{url}            % simple URL typesetting
\usepackage{booktabs}       % professional-quality tables
\usepackage{amsfonts}       % blackboard math symbols
\usepackage{nicefrac}       % compact symbols for 1/2, etc.
\usepackage{microtype}      % microtypography
\usepackage{mathtools}
\usepackage{natbib}
\usepackage{xcolor}
\usepackage{amsmath}
\usepackage{mathtools}
\usepackage{amsthm}
\usepackage{amssymb}
\usepackage{enumitem}
\usepackage{bm}

\definecolor{Green}{rgb}{0.13, 0.65, 0.3}
\definecolor{Amber}{rgb}{0.3, 0.5, 1.0}
\usepackage[]{color-edits}
\addauthor{HL}{Green}
\addauthor{LC}{Amber}

\usepackage[algo2e, ruled, vlined]{algorithm2e}
\usepackage{algorithm}
\SetKwBlock{MyRepeat}{repeat}{}

\newcommand{\calA}{{\mathcal{A}}}

\newcommand{\calV}{{\mathcal{V}}}
\newcommand{\calW}{\mathcal{W}}

\newcommand{\calS}{{\mathcal{S}}}
\newcommand{\calF}{{\mathcal{F}}}

\newcommand{\calI}{{\mathcal{I}}}

\newcommand{\calT}{{\mathcal{T}}}
\newcommand{\calP}{{\mathcal{P}}}
\newcommand{\calQ}{{\mathcal{Q}}}

\newcommand{\calM}{{\mathcal{M}}}

\DeclareMathOperator*{\argmin}{argmin}
\DeclareMathOperator*{\argmax}{argmax}

\newcommand{\eat}[1]{}

\newcommand{\inner}[2]{\left\langle #1, #2 \right\rangle}

\newcommand{\rbr}[1]{\left(#1\right)}
\newcommand{\sbr}[1]{\left[#1\right]}
\newcommand{\cbr}[1]{\left\{#1\right\}}

\newcommand{\abr}[1]{\left|#1\right|}

\newcommand{\tilO}[1]{\otil\left( #1 \right)}

\newcommand{\bigo}[1]{\order( #1 )}
\newcommand{\tilo}[1]{\otil( #1 )}

\DeclarePairedDelimiter\ceil{\lceil}{\rceil}

\newcommand{\dev}{\textsc{Dev}}
\newcommand{\reg}{\textsc{Reg}}
\newcommand{\diff}{\textsc{Diff}}

\newcommand{\var}{\textsc{Var}}

\newcommand{\SA}{\calS\times\calA}%{\Gamma}
\renewcommand{\sp}{\text{\rm sp}}
\renewcommand{\P}{\bar{P}}

\newcommand{\istar}{i^{\star}}
\newcommand{\Np}{\N^+}
\newcommand{\optJ}{J^{\star}}

\newcommand{\sumt}{\sum_{t=1}^T}

\newcommand{\hatQ}{\widehat{Q}}
\newcommand{\hatV}{\widehat{V}}

\newcommand{\sums}{\sum_{s\in\calS}}
\newcommand{\suma}{\sum_{a\in\calA}}

\newcommand{\sumsa}[1][s, a]{\sum_{#1}}

\newcommand{\sumh}{\sum_{h=1}^H}

\newcommand{\sumk}{\sum_{k=1}^K}

\newcommand{\optpi}{\pi^\star}

\newcommand{\hatJ}{\widehat{J}}

\newcommand{\tilP}{\widetilde{P}}

\newcommand{\tilpi}{{\widetilde{\pi}}}
\newcommand{\tiloptpi}{{\widetilde{\pi}^\star}}

\newcommand{\N}{N} % counter N

\newcommand{\epsevi}{\epsilon_{\EVI}}
\newcommand{\piref}{\mathring{\pi}}

\newcommand{\sstar}{s^{\star}}

\newcommand{\estQ}{\textsc{EstimateQ}\xspace}

\newcommand{\optpieps}{\pi^{\star,\epsilon}}
\newcommand{\optJeps}{J^{\star,\epsilon}}
\newcommand{\tmix}{t_{\text{mix}}}
\newcommand{\thit}{t_{\text{hit}}}

\newcommand{\tilcalM}{\widetilde{\calM}}

\newcommand{\EVI}{\textsc{EVI}}
\newcommand{\hatbeta}{\widehat{\beta}}

\newcommand{\field}[1]{\mathbb{#1}}

\newcommand{\fR}{\field{R}}

\newcommand{\fN}{\field{N}}
\newcommand{\E}{\field{E}}
\newcommand{\fV}{\field{V}}

\newcommand{\Ind}{\field{I}}

\newcommand{\norm}[1]{\left\|{#1}\right\|}

\newtheorem{lemma}{Lemma}
\newtheorem{theorem}{Theorem}

\newcommand{\order}{\ensuremath{\mathcal{O}}}

\newcommand{\otil}{\ensuremath{\tilde{\mathcal{O}}}}

\usepackage{prettyref}
\newcommand{\pref}[1]{\prettyref{#1}}
\newcommand{\pfref}[1]{Proof of \prettyref{#1}}
\newcommand{\savehyperref}[2]{\texorpdfstring{\hyperref[#1]{#2}}{#2}}
\newrefformat{eq}{\savehyperref{#1}{Eq.~\textup{(\ref*{#1})}}}
\newrefformat{eqn}{\savehyperref{#1}{Equation~\ref*{#1}}}
\newrefformat{lem}{\savehyperref{#1}{Lemma~\ref*{#1}}}
\newrefformat{def}{\savehyperref{#1}{Definition~\ref*{#1}}}
\newrefformat{line}{\savehyperref{#1}{Line~\ref*{#1}}}
\newrefformat{thm}{\savehyperref{#1}{Theorem~\ref*{#1}}}
\newrefformat{corr}{\savehyperref{#1}{Corollary~\ref*{#1}}}
\newrefformat{cor}{\savehyperref{#1}{Corollary~\ref*{#1}}}
\newrefformat{sec}{\savehyperref{#1}{Section~\ref*{#1}}}
\newrefformat{subsec}{\savehyperref{#1}{Section~\ref*{#1}}}
\newrefformat{app}{\savehyperref{#1}{Appendix~\ref*{#1}}}
\newrefformat{assum}{\savehyperref{#1}{Assumption~\ref*{#1}}}
\newrefformat{ex}{\savehyperref{#1}{Example~\ref*{#1}}}
\newrefformat{fig}{\savehyperref{#1}{Figure~\ref*{#1}}}
\newrefformat{alg}{\savehyperref{#1}{Algorithm~\ref*{#1}}}
\newrefformat{rem}{\savehyperref{#1}{Remark~\ref*{#1}}}
\newrefformat{conj}{\savehyperref{#1}{Conjecture~\ref*{#1}}}
\newrefformat{prop}{\savehyperref{#1}{Proposition~\ref*{#1}}}
\newrefformat{proto}{\savehyperref{#1}{Protocol~\ref*{#1}}}
\newrefformat{prob}{\savehyperref{#1}{Problem~\ref*{#1}}}
\newrefformat{claim}{\savehyperref{#1}{Claim~\ref*{#1}}}
\newrefformat{que}{\savehyperref{#1}{Question~\ref*{#1}}}
\newrefformat{op}{\savehyperref{#1}{Open Problem~\ref*{#1}}}
\newrefformat{fn}{\savehyperref{#1}{Footnote~\ref*{#1}}}
\newrefformat{tab}{\savehyperref{#1}{Table~\ref*{#1}}}
\newrefformat{fig}{\savehyperref{#1}{Figure~\ref*{#1}}}
\newrefformat{prop}{\savehyperref{#1}{Property~\ref*{#1}}}

%% file: intro.tex
% !TEX root = main.tex

Standard reinforcement learning (RL) algorithms aim at finding the optimal policy that maximizes the accumulated reward in a Markov Decision Process (MDP).
%
%Reinforcement Learning (RL) solves the problem of finding the optimal control policy that maximizes some accumulated reward when interacting with an unknown environment.
%Specifically, a learner interacts with a Markov Decision Process (MDP) for $T$ steps.
%In each time step, the learner observes a state, takes an action, receives a reward, and %transits to the next state following the transition dynamic.
%
%In the standard RL setting, the learner only focuses on maximizing a single reward function.
In many real-world applications, however, the algorithm is also required to satisfy certain constraints.
For example, in autonomous driving, the vehicle needs to reach the destination with minimum amount of time while obeying the traffic rules.
These constrained versions of RL problems can be formulated by Constrained Markov Decision Processes (CMDPs)~\citep{altman1999constrained}, where a learning agent tries to maximize the accumulated reward while ensuring that certain cost constraint is not violated or at least the violation is small enough.

Learning in a CMDP is a long-standing topic, and there is a surge of interest in it in light of all other theoretical advances in RL.
Almost all recent works on CMDP, however, focus on the simpler finite-horizon setting~\citep{kalagarla2020sample,efroni2020exploration,qiu2020upper,liu2021learning} or the discounted setting~\citep{liang2018accelerated,tessler2018reward,chen2021primal,liu2021fast}.
In contrast, learning CMDP in the infinite-horizon average-reward setting, where the learner-environment interaction never ends or resets and the goal is to achieve optimal long-term average reward under constraints, 
appears to be much more challenging.
For example, \citep{zheng2020constrained} makes the restricted assumptions that the
transition kernel is known and an initial policy that satisfies the constraints and induces an ergodic Markov chain is given,
but still only achieves $\tilo{T^{3/4}}$ regret after $T$ steps (with no constraint violation).
Another recent work~\citep{singh2020learning} considers the special class of ergodic CMDPs, but only achieves $\tilo{T^{2/3}}$ regret and $\tilo{T^{2/3}}$ cost constraint violation.
These existing results are far from optimal, exhibiting unique challenges of the constrained infinite-horizon average-reward setting.

In this work, we manage to overcome some of these challenges and significantly improve our understanding of regret minimization for infinite-horizon average-reward CMDPs.
Our contributions are as follows:
\begin{itemize}[leftmargin=*]
  \setlength\itemsep{0em}
	\item Following~\citep{singh2020learning}, we start by considering ergodic CMDPs in \pref{sec:ergodic}. We develop an algorithm that achieves $\tilo{\sqrt{T}}$ regret and constant constraint violation, strictly improving~\citep{singh2020learning}.
	The main technical challenge in getting $\tilo{\sqrt{T}}$ regret for the upper-confidence-type algorithm of~\citep{singh2020learning} is the lack of a tighter bound on the span of the estimated bias function.
	Instead, we resolve this issue using a policy optimization algorithm with a special action-value estimator whose span is well controlled.
	To further control the transition estimation error from the action-value estimator, we also include a new bonus term in the policy update.
	
	\item In \pref{sec:weak}, we drop the ergodic assumption and consider the most general class of weakly communicating CMDPs. 
	By reducing the original infinite-horizon problem to a finite-horizon problem similarly to~\citep{wei2021learning},
	we show that a simple and efficient linear programming approach gives $\tilo{T^{2/3}}$ regret and $\tilo{T^{2/3}}$ constraint violation.
	Further introducing extra constraints to the linear program to control the span of some bias function, we are also able to obtain $\tilo{\sqrt{T}}$ regret and $\tilo{\sqrt{T}}$ violation, with the price that the resulting program can no longer be solved computationally efficiently.
	As far as we know, these are the first results for weakly communicating CMDPs (see some caveats below).
\end{itemize}

\paragraph{Related Work}
As mentioned, learning in CMDP is heavily studied recently in other settings (see references listed earlier),
but for the infinite-horizon average-reward setting, other than the two recent works discussed above~\citep{zheng2020constrained, singh2020learning}, 
we are only aware of~\citep{agarwal2021concave,agarwal2021markov} which also study both the ergodic case and the weakly communicating case. 
Unfortunately, their results appear to be wrong due to a technical mistake which sidesteps an important challenge for this problem on controlling the span of some bias function; see \pref{app:pre} for more details.

%\footnote{We also realize two recent works~\citep{agarwal2021markov,agarwal2021concave} that claim to achieve $\tilo{\text{poly}(DSA)\sqrt{T}}$ regret under the ergodic assumption and the weakly communicating assumption respectively, where $D$ is the diameter. However, we spot some issues in their proof, which are confirmed by the authors. Specifically, they wrongly assume that the span of a bias function w.r.t some estimated transition function can be directly bounded by $D$. See \pref{app:issue} for more details.}

%Unconstrained RL problems
%Unconstrained RL problems have been extensively studied in various MDP models, such as finite-horizon model~\citep{azar2017minimax,jin2018q}, discounted model~\citep{even2003learning,strehl2006pac}, average-reward model~\citep{bartlett2012regal,jaksch2010near}, and stochastic shortest path~\cite{tarbouriech2020no,cohen2020near,chen2021implicit}.
%There are plenty work concerning regret minimization in unconstrained average-reward MDPs~\citep{jaksch2010near,}.

Regret minimization for the infinite-horizon average-reward setting without constraints dates back to~\citep{bartlett2009regal,jaksch2010near} and was shown to be possible only when the MDP is at least weakly communicating.
Numerous improvements have been discovered in recent years; see e.g.~\citep{ortner2018regret, fruit2018efficient, talebi2018variance, abbasi2019politex, zhang2019regret, wei2020model, wei2021learning}.
From a technical perspective, designing provable algorithms for the infinite-horizon average-reward setting, especially for the general class of weakly communicating MDPs, has always been more challenging than other settings.
For example, optimal model-free algorithms remain unknown for this setting~\citep{wei2020model}, but have been developed for the finite-horizon setting~\citep{jin2018q} and the discounted setting~\citep{dong2020q}.

Apart from MDPs, researchers also study constrained multi-armed bandit problems, such as conservative bandits~\cite{wu2016conservative,kazerouni2016conservative,garcelon2020improved} and bandits with safety constraints modeled by a cost function (similar to our setting)~\citep{amani2019linear,pacchiano2021stochastic,liu2021efficient}.
%\cite{xx} initiate the study on constrained average reward MDPs and it is the only valid result\footnote{\cite{xx} also study this setting but their results is wrong, see Appendix for details}.

%% file: pre.tex
% !TEX root = main.tex

An infinite-horizon average-reward CMDP model is defined as a tuple $\calM=(\calS, \calA, r, c, \tau, P)$.
Here, $\calS$ is the state space, $\calA$ is the action space, $r\in [0, 1]^{\SA}$ is the reward function, $c\in[0, 1]^{\SA}$ is the cost function modeling constraints, $\tau$ is a cost threshold, and $P=\{P_{s, a}\}_{(s, a)\in\SA}$ with $P_{s, a}\in\Delta_{\calS}$ is the transition function, where $\Delta_{\calS}$ is the simplex over $\calS$.
For simplicity, we assume that only the transition function $P$ is unknown, while all other parameters are known.
Dealing with unknown reward and cost functions can be done in a way similar to~\citep{liu2021learning} by maintaining standard confidence sets.

%Before discussing the learning protocol and objective, we introduce an important concept called occupancy measure.
%Given a (stationary) policy $\pi\in\calA^{\calS}$ and a transition function $P$, we define its occupancy measure $q_{\pi,P}\in [0, 1]^{\SA}$ such that $q_{\pi,P}(s, a)=\lim_{T\rightarrow\infty}\frac{1}{T}\E[\sumt \Ind\{s_t=s,a_t=t\}]$ is the fraction of visits to $(s, a)$ in the long following policy $\pi$ in an MDP with transition $P$.

Throughout, we also assume that the MDP is \emph{weakly communicating}, which is known to be necessary for learning even without constraints~\citep{bartlett2009regal}. 
More specifically, an MDP is weakly communicating if its state space consists of two subsets: in the first subset, all states are transient under any stationary policy (that is, a mapping from $\calS$ to $\Delta_{\calA}$);
in the second subset, every two state are communicating under some stationary policy.

The learning protocol is as follows: the learner starts from an arbitrary state $s_1\in\calS$, and interacts with the environment for $T$ steps.
In the $t$-th step, the learner observes state $s_t\in\calS$, takes an action $a_t\in\calA$, and transits to the next state $s_{t+1}\sim P_{s_t,a_t}$.
Informally, the goal of the learner is to ensure large reward while at the same time incurring small cost relative to the threshold $\tau$.

To describe these objectives formally, we introduce the concept of average utility function:
for a stationary policy $\pi \in (\Delta_{\calA})^{\calS}$, transition function $P$, and utility function $d\in\fR_+^{\SA}$, define the average utility for any $s\in\calS$ as
$$J^{\pi,P,d}(s)=\liminf_{T\rightarrow\infty}\frac{1}{T}\E\sbr{\left.\sumt d(s_t, a_t)\right|\pi,P,s_1=s}$$
where the expectation is with respect to the random sequence $a_1, s_2, a_2, s_3, a_3, \ldots$ generated according to $a_t \sim \pi$ and $s_{t+1} \sim P_{s_t,a_t}$.
\citep[Theorem 8.3.2]{puterman1994markov} shows that, there exists an optimal policy $\optpi$ such that for any $s\in\calS$, $\optpi$ is the solution for the following optimization problem
\begin{equation}
	\label{eq:obj}
	\argmax_{\pi\in(\Delta_{\calA})^{\calS}}J^{\pi, P, r}(s), \quad \text{s.t. } J^{\pi,P,c}(s)\leq\tau,
\end{equation}
and also $J^{\pi^\star, P, r}(s) = \optJ$ and $J^{\pi^\star,P,c}(s) = \optJ_c$ for some constants $\optJ$ and $\optJ_c$ independent of $s$.
The performance of the learner is then measured through two quantities: first, her regret in reward against the optimal policy $\optpi$, defined as $R_T = \sumt (\optJ - r(s_t, a_t))$, and second, her regret in cost against the threshold $\tau$, or simply her constraint violation, defined as $C_T = \sumt (c(s_t, a_t) - \tau)$.

Finally, \citep[Theorem 8.2.6]{puterman1994markov} also shows that for any utility function $d$,  there exists a \emph{bias function} $q^{\pi,P,d}\in\fR^{\SA}$ satisfying the Bellman equation: $\forall (s,a)\in\calS \times \calA$,
\begin{equation}\label{eq:Bellman}
q^{\pi,P,d}(s, a)+J^{\pi,P,d}(s)=d(s,a)+\E_{s'\sim P_{s, a}}[v^{\pi,P,d}(s')],
\end{equation}
where $v^{\pi,P,d}(s)=\sum_{a\in\calA}\pi(a|s)q^{\pi,P,d}(s, a)$,
and also $J^{\pi,P,q^{\pi,P,d}}(s)=0$ for all $s\in\calS$.
The functions $q$ and $v$ are analogue of the well-known $Q$-function and state-value-function for the discounted or finite-horizon setting. 

%By \citep[Theorem 8.3.2]{puterman1994markov}, we also assume that $\optpi$ is the ``best'' optimal policy (among all solutions of \pref{eq:obj}) in the following sense: there exist constants $J^{\optpi,P,r}$ and $J^{\optpi,P,c}$ such that $J^{\optpi,P,r}(s)=J^{\optpi,P,r}$ and $J^{\optpi,P,c}(s)=J^{\optpi,P,c}$ for any $s\in\calS$.
%Moreover, $J^{\pi,P,r}(s)\leq J^{\optpi,P,r}$ and $J^{\pi,P,c}(s)\geq J^{\optpi,P,c}$ for any $s\in\calS$ and $\pi$ that is a solution of \pref{eq:opt}.
%The performance of the learner is then measured through two quantities against $\optpi$: her regret $R_T = \sumt (\optJ - r(s_t, a_t))$ where $\optJ=J^{\optpi,P,r}$, and her constraint violation $C_T = \sumt (c(s_t, a_t) - \tau)$.

\paragraph{Notations} 
Let $S = |\calS|$ and $A = |\calA|$ be the number of states and actions respectively.
For an integer $n$, $[n]$ denotes the set $\{1, \ldots, n\}$.
For a distribution $P\in\Delta_{\calS}$ and a function $V\in \fR^{\calS}$, define $PV=\sum_{s\in\calS}P(s)V(s)$.
For any function $v\in\fR^{\calS}$, define its span as $\sp(v)=\max_{s\in\calS}v(s) - \min_{s\in\calS}v(s)$.
When there is no confusion, we write $J^{\pi,P,d}$ as $J^{\pi,d}$, $q^{\pi,P,d}$ as $q^{\pi,d}$, and $v^{\pi,P,d}$ as $v^{\pi,d}$.
Given a policy $\pi$ and a transition $P$, define matrix $P^{\pi}$ such that $P^{\pi}_{s,s'}=\sum_a\pi(a|s)P_{s, a}(s')$.
For any $\epsilon\in(0, 1)$, $\optpieps$ is defined in the same way as $\optpi$ but with the threshold $\tau$ replaced by $\tau - \epsilon$. Let $\optJeps$ denote the corresponding average reward $J^{\optpieps,r}(s)$ (that is $s$-independent as mentioned).

%% file: ergodic.tex
% !TEX root = main.tex

We start by considering a special case of ergodic MDPs, which are self-explorative and often easier to learn compared to the general case of weakly communicating MDPs.
However, even in this special case, the presence of cost constraint already makes the problem highly challenging as discussed below.

Specifically, an MDP is ergodic if for any stationary policy, the induced Markov chain is ergodic (that is, irreducible and aperiodic).
There are several nice properties about ergodic MDPs.
First, the long term average behavior of any stationary policy $\pi$ is independent of the starting state: one can define the occupancy measure (also called stationary distribution) $\mu_{\pi,P}\in[0, 1]^{\SA}$ such that $\mu_{\pi,P}(s, a)=\lim_{T\rightarrow\infty}\frac{1}{T}\E[\sumt\Ind\{s_t=s,a_t=a\}|\pi,P]$ is the fraction of visits to $(s, a)$ in the long run following $\pi$ in an ergodic MDP with transition $P$ (the starting state is irrelevant to this value).
We also write $\mu_{\pi,P}$ as $\mu_{\pi}$ when there is no confusion.
By definition, for any utility function $d\in\fR^{\SA}$, the average utility $J^{\pi,d}(s)$ is also $s$-independent and can be written as $\inner{\mu_{\pi}}{d}$, denoted by $J^{\pi,d}$ for short.

Moreover, ergodic MDPs have finite \textit{mixing time} and \textit{hitting time}, defined as follows:
\begin{align*}
	\tmix &= \max_{\pi}\min\cbr{t\geq 1: \norm{(P^{\pi})^t_{s,\cdot} - \mu_{\pi}}_1 \leq \frac{1}{4},\forall s }, \\
	\thit &= \max_{\pi}\max_s\frac{1}{\mu_{\pi}(s)},
\end{align*}
where $\mu_{\pi}(s)=\sum_a\mu_{\pi}(s, a)$.
In words, mixing time is the maximum time required for any policy starting
at any initial state to make the state distribution $\frac{1}{4}$-close to its stationary distribution,
and hitting time is the maximum inverse stationary probability of visiting any state under any policy.
As in previous work, we assume that $\tmix$, $\thit$ are known, and $T$ is large enough so that $T\geq 30A\max\{\tmix,\thit\}$.

Compared to the finite-horizon setting, one key challenge for learning infinite-horizon average-reward MDPs with constraints is to control the span of the bias function with respect to some estimated transition. 
How to do so is highly unclear even though the same under the true transition is simply bounded by $\bigo{\tmix}$ for an ergodic MDP.
In fact, the MDP associated with the estimated transition might not be ergodic any more.
In the seminal work of~\citep{jaksch2010near} for the unconstrained problem, they show that the span of the optimistic bias function is upper bounded by the diameter of the MDP.
Their arguments, however, are not applicable when constraints are presented.
This brings severe difficulties in the analysis for natural optimism-based approaches.
For example, \citet{singh2020learning} exploit the self-explorative property of ergodic MDPs to analyze an extension of the UCRL algorithm~\citep{jaksch2010near}, but only manage to obtain $\tilo{T^{2/3}}$ bounds for both regret $R_T$ and constraint violation $C_T$.
Moreover, their analysis does not generalize to the case of weakly communicating MDPs.

%Learning infinite-horizon average-reward MDPs with constraints exhibits unique challenges compared to the finite-horizon MDPs.
%The main issue lies in bounding the span of the bias function w.r.t some estimated transition dynamic.
%Note that while the span of any stationary policy is nicely bounded under the true transition $P$, it may not be the case under the estimated transition.
%In fact, the estimated transition may not be ergodic.
%In \cite{jaksch2010near}, they show that the span of the optimistic bias function is upper bounded by the diameter.
%Their arguments, however, are not applicable with constraints enforced.
%This brings severe difficulties in the standard analysis of the optimism based approach.
%In \citep{singh2020learning}, exploiting the self-explorative property of ergodic MDPs, they analyze a UCRL-like algorithm by considering the number of steps needed to have sufficient number of visits to all state-action pairs.
%However, their analysis does not generalize to the weakly communicating assumption, and only obtains sub-optimal regret bound.

\subsection{Our Algorithm}\label{sec:ergodic_alg}
To resolve the issue mentioned above, we take a different approach --- we adopt and extend the \emph{policy optimization} algorithm of~\citep{wei2020model} (called MDP-OOMD) from the unconstrained setting to the constrained one.
The advantage of policy optimization is that, instead of finding optimistic policies and transitions based on full planning, it updates policy incrementally based on an estimate of the current policy's bias function, which avoids the need to control the span of bias functions under estimated transition.
This, however, requires a careful design of the estimate of the current policy's bias function, which is our key algorithmic novelty.

We start by describing the framework of our algorithm, which is similar to~\citep{wei2020model}, and then highlight what the key differences are.
The complete pseudocode is presented in \pref{alg:ergodic}.
Specifically, the algorithm proceeds in episodes of $H= \tilo{\tmix\thit}$ steps, with a total of $K=\frac{T}{H}$ episodes (assumed to be an integer for simplicity).
In each episode $k$, the algorithm (\pref{line:execute}) executes the same policy $\pi_k$ for the entire $H$ steps, collecting a trajectory $\calT_k$ of the form $(s_{t_1}, a_{t_1},\ldots, s_{t_2}, a_{t_2})$ for $t_1 = (k-1)H+1$ and $t_2 = kH$.
Then, using this trajectory (together with other statistics), the algorithm (\pref{line:estimate_beta}) invokes a procedure $\estQ$ to compute a bias function estimator $\hatbeta_k \in \fR^{\calS\times\calA}$, such that $\hatbeta_k(s,a)$ approximately tells us how good taking action $a$ at state $s$ and then following $\pi_k$ in the future is.
With such an estimator, the algorithm (\pref{line:update policy}) updates the policy and find $\pi_{k+1}$ for the next episode, using the classic Online Mirror Descent (OMD) framework.
Below, we flesh out the details of each part.

\DontPrintSemicolon 
\setcounter{AlgoLine}{0}
\begin{algorithm}[t]
	\caption{Policy Optimization for Ergodic CMDP}
	\label{alg:ergodic}
	\textbf{Parameter:} episode length $H$, number of episodes $K=T/H$, interval length $N$, learning rate $\theta$, scaling parameter $\eta$, dual variable upper bound $\lambda$, cost slack $\epsilon$; see \pref{eq:parameters}.
	
	\textbf{Initialize:} $\pi_1(a|s)=1/A$ for $(s, a)\in\SA$, $\lambda_1=0$.
	%$\epsilon=\min\{(\tau-c^0)/2, 3\lambda/K\}$.
	%$\theta=\eta/( \lambda NH\sqrt{SK} )$, $\epsilon=\lambda/K$.
	
	\For{$k=1,\ldots,K$}{
		\nl Execute $\pi_k$ for $H$ steps and obtain trajectory $\calT_k$. \label{line:execute}
		
		\nl $\hatbeta_k=\estQ(\calT_k, \P_k, r - \frac{\lambda_k}{\eta}c)$ (\pref{alg:estQ}) where $\P_k$ is empirical transition~\eqref{eq:empirical_transition}. \label{line:estimate_beta}
		
		\For{all $s\in\calS$}{
			\nl Update policy: \label{line:update policy}
			\begin{align*}
				&\pi_{k+1}(\cdot|s)
				= \argmax_{\pi\in\bar{\Delta}}\cbr{\inner{\pi}{\hatbeta_k + u_k} - D(\pi, \pi_k)}
			\end{align*}
			where $D(p, q)=\frac{1}{\theta}\sum_a(p(a)\ln\frac{p(a)}{q(a)} - p(a)+q(a))$, $\bar{\Delta}=\Delta_\calA \cap [\frac{1}{T}, 1]^{\calA}$, and $u_k$ is defined in \pref{app:Pk} (see also \pref{sec:ergodic_alg}). 
			
			\nl Update dual variable 
			\[
			\lambda_{k+1} = \min\cbr{\lambda, \max\cbr{0, \lambda_k + \hatJ_k + \epsilon - \tau}},
			\] where $\hatJ_k=\frac{1}{H-N}\sum_{h=N+1}^{H}c(s^k_h, a^k_h)$ with $s^k_h = s_{(k-1)H+h}$ and $a^k_h = a_{(k-1)H+h}$. \label{line:dual}
		}
	}
	
\end{algorithm}

\paragraph{Bias Function Estimates}
To simultaneously take reward and cost constraint into account, we adopt the common primal-dual approach~\citep{efroni2020exploration,liu2021learning} and consider the adjusted reward function $d = r - \frac{\lambda_k}{\eta}c$ for episode $k$, where $\eta$ is a scaling parameter and $\lambda_k$ is a dual variable (whose update will be discussed later).
Intuitively, this adjusted reward provides a balance between maximizing rewards and minimizing costs.
The procedure $\estQ$ is effectively trying to estimate the bias function associated with the current policy $\pi_k$ and the adjusted reward $d$, that is, $q^{\pi_k, d}$ (up to some additive term that is the same across all $(s,a)$ entries), using data from the trajectory $\calT_k$ of this episode.
The pseudocode of $\estQ$ in shown in \pref{alg:estQ}.
It shares a similar framework as~\citep{wei2020model}: 
for each $s\in\calS$, it collects data from non-overlapping intervals of length $N = \tilo{\tmix}$ that start from state $s$, and also make sure that these intervals are at least $N$ steps apart from each other to reduce the correlation of data (see the while-loop of \pref{alg:estQ}).

However, different from~\citep{wei2020model} which uses standard importance weighting when constructing the estimator, we propose a new method that is critical to our analysis (see last two lines of \pref{alg:estQ}).
Specifically, for each interval $i$ mentioned above, we compute the cumulative adjusted reward $y_i$, and then average them over all intervals starting from state $s$ as a value-function estimate $V(s)$.
Finally, we return the bias function estimate whose $(s,a)$ entry is $d(s,a) + \P_k V$, in light of the right-hand side of the Bellman equation~\eqref{eq:Bellman}. Here, $\P_k$ is the current empirical transition function such that
\begin{equation}\label{eq:empirical_transition}
\P_{k,s,a}(s')=\frac{\N_k(s, a, s')}{\Np_k(s, a)},
\end{equation}
where $\N_k(s, a, s')$ is the number of visits to state-action-state triplet $(s, a, s')$ before episode $k$ and $\Np_k(s,a)=\max\{1,\N_k(s, a)\}$ with $\N_k(s,a) = \sum_{s'}\N_k(s, a, s')$.

The reason of using this new estimator is to ensure that the final estimator $\hatbeta_k(s, a)$ has a reasonable scale (roughly $\bigo{N} = \tilo{\tmix}$), which in turn makes sure that the policy $\pi_k$ is relatively stable across episodes.
On the other hand, the importance-weighted estimator of~\citep{wei2020model} scales with $\frac{1}{\pi_k(a|s)}$ for the $(s,a)$ entry, which could be very large.
This is not an issue for their algorithm since they use a more stable regularizer called log-barrier in the OMD policy update, but this is not viable for us as explained next.

\DontPrintSemicolon
\setcounter{AlgoLine}{0}
\begin{algorithm}[t]
	\caption{\estQ}
	\label{alg:estQ}
	\textbf{Input:} trajectory $\calT=(s_{t_1}, a_{t_1},\ldots, s_{t_2}, a_{t_2})$, empirical transition $\P$, and utility function $d$.
	
	\textbf{Define:} $N=4\tmix\log_2T$.
	
	\For{all $s\in\calS$}{
		\textbf{Initialize:} $\tau\leftarrow t_1, i \leftarrow 0$.
	
		\While{$\tau\leq t_2-N$}{
			\If{$s_{\tau}=s$}{
				$i\leftarrow i+1$.
				
				$y_i=\sum_{t=\tau}^{\tau+N-1}d(s_t, a_t)$.
				
				$\tau\leftarrow\tau+2N$.
			}
			\lElse{$\tau\leftarrow\tau+1$.}
		}
		
		Set $V(s)=\Ind\{i>0\}\frac{1}{i}\sum_{j=1}^iy_j$. %\label{line:compute V}
		%\lIf{$i>0$}{\textbf{return} $\frac{1}{i}\sum_{j=1}^iy_j$ \textbf{else} \textbf{return} $0$.}		
	}
	
	 \textbf{return} function $Q$ such that $Q(s, a)=d(s, a) + \P V$.%\label{line:compute q}
		
\end{algorithm}

\paragraph{Policy Update}
With the estimate $\hatbeta_k$, \pref{alg:ergodic} then updates $\pi_k$ according to the OMD update: $\pi_{k+1}(\cdot|s) =\argmax_{\pi\in\bar{\Delta}}\cbr{\inner{\pi}{\hatbeta_k + u_k} - D(\pi, \pi_k)}$.
Here, $\bar{\Delta} =\Delta_\calA \cap [\frac{1}{T}, 1]^{\calA}$ is a truncated simplex,  $D(p, q)=\frac{1}{\theta}\sum_a(p(a)\ln\frac{p(a)}{q(a)} - p(a)+q(a))$ is the KL-divergence scaled by the inverse of a learning rate $\theta>0$,
and finally $u_k$ is an extra exploration bonus term used to cancel the bias due to using $\P_k$ instead of the true transition $P$ when computing $\hatbeta_k$.
More concretely, $u_k$ is an approximation of $q^{\pi_k,P_k,x_k}$,
where $x_k$ is a reward bonus function defined as
\begin{equation}
	\label{eq:xk}
	x_k(s, a) = \rbr{\frac{1}{\sqrt{\Np_k(s, a)}} + \sum_{a'}\frac{\pi_k(a'|s)}{\sqrt{\Np_k(s, a')}}}\iota,
\end{equation}
for $\iota=\frac{2\lambda N}{\eta}\sqrt{S\ln\frac{2SAT}{\delta}}$ and failure probability $\delta \in (0,1)$, and $P_k$ is an optimistic transition with respect to policy $\pi_k$ and reward bonus $x_k$ that lies in some transition confidence set.
How to exactly compute $P_k$ and $u_k$ via Extended Value Iteration (EVI)~\citep{jaksch2010near} with precision $\epsevi=\frac{1}{T}$ is deferred to \pref{app:Pk} due to space limit.

If $\bar{\Delta}$ is replaced with the standard simplex $\Delta_\calA$, then the update rule would be in the standard form of multiplicative weight update.
However, we use a truncated simplex because our analysis requires a so-called ``interval regret'' guarantee from OMD (that is, regret measured on a specific interval; see the proof of \pref{lem:bound lambda}), and multiplicative weight update over a truncated simplex is a standard way to achieve so. 
This is why we cannot use the log-barrier OMD of~\citep{wei2020model} because as far as we know it does not provide an interval regret guarantee.

\paragraph{Dual Variable Update}
The dual variable $\lambda_{k+1}$ is updated via gradient descent $\lambda_k + \hatJ_k + \epsilon - \tau$, projected back to $[0, \lambda]$; see \pref{line:dual} of \pref{alg:ergodic}.
Here, $\epsilon$ is a cost slack, $\lambda$ is an upper bound for the dual variable, and 
$\hatJ_k$ is the empirical average cost suffered by the learner in the last $H-N$ steps of this episode (discarding the first $N$ steps to ensure that the state-action distribution is close to $\mu_{\pi_k}$ due to ergodicity).
It can be shown that $\hatJ_k$ is an accurate estimate of $J^{\pi_k,c}$; see \pref{lem:hatJ}.

\paragraph{Parameter Tuning}
Finally, we list the exact value of the parameters below (recall $\iota=\frac{2\lambda N}{\eta}\sqrt{S\ln\frac{2SAT}{\delta}}$):
\begin{equation}\label{eq:parameters}
\begin{split}
&H=\lceil 16\tmix\thit(\log_2T)^2\rceil, \quad N=\lceil 4\tmix\log_2T \rceil, \\
&\theta=\min\cbr{1/(4H\iota), \sqrt{\ln T/(4KH^2\iota^2)}}, \\
&\eta=1+2^{10}(\tau-c^0)N\sqrt{S}\ln\rbr{\frac{4SAT^3}{\delta}} \times \\
& \rbr{\sqrt{S^2AT} + \sqrt{HT} + S^{1.5}AH\ln\rbr{\frac{4SAT^3}{\delta})}}, \\
&\lambda=\frac{40\eta}{\tau-c^0}, \quad \epsilon=\min\cbr{\frac{\tau-c^0}{2}, \frac{3\lambda}{K}}.
\end{split}
\end{equation}
Here, $c^0$ is a constant such that there is a (strictly safe) policy $\pi^0$ with $c^0= \max_sJ^{\pi^0,P,c}(s) < \tau$.
We assume that $c^0$ is known (but not the policy $\pi^0$), similarly to prior work such as~\citep{efroni2020exploration,liu2021learning}.
%and this is mainly to ensure constant constraint violation.
%Without knowing $c^0$ or the existence of such a strictly safe policy, our algorithm could still achieve $\tilo{\sqrt{T}}$ bounds for both regret and constraint violation (details omitted).
%Regarding the complicated dentition of $\eta$, we point out that one only needs the fact $\frac{\lambda}{\eta}=\bigo{\frac{1}{\tau-c^0}}$ to understand the high level idea of the analysis.

\subsection{Guarantees and Analysis}
The guarantee of \pref{alg:ergodic} is summarized below.
\begin{theorem}
	\label{thm:ergodic}
	With probability at least $1-16\delta$, \pref{alg:ergodic} ensures the following guarantees: $R_T = \tilO{ \frac{\tmix^2\thit}{\tau-c^0}(\sqrt{S^3AT} + \sqrt{S\tmix\thit T} + S^2A\tmix\thit) + \frac{\tmix\thit}{(\tau-c^0)^2} }$ and $C_T=\tilO{\frac{\tmix^4\thit^2S^3A + \tmix^5\thit^3S}{(\tau-c^0)^2} + \frac{\tmix^3\thit^2S^2A}{\tau-c^0}}$.
\end{theorem}

Looking only at the dependence on $T$, our bounds are $R_T = \tilo{\sqrt{T}}$ and $C_T = \tilo{1}$, improving the $\tilo{T^{2/3}}$ bounds of~\citep{singh2020learning} for both metrics. Below, we show a proof sketch of this theorem.
%Our bound improves over \cite{singh2020learning} in dependency on $T$.
%There is still a large dependency on $S$, $\tmix$, and $\thit$, which we think can be improved.
%Also note that all existing mirror-descent-based algorithms for the average reward setting has dependency on $\tmix$.

\paragraph{Regret} For regret $R_T$, we start by decomposing it as
\begin{align*}
	R_T&=\sumt (\optJ - r(s_t, a_t))
	= \underbrace{T(\optJ - \optJeps)}_{\diff} + \\
	& H\underbrace{\sumk (\optJeps - J^{\pi_k,r})}_{\reg}
 + \underbrace{\sumk\sum_{h=1}^H (J^{\pi_k,r} - r(s^k_h, a^k_h))}_{\dev},
\end{align*}
where $s^k_h = s_{(k-1)H+h}$ and $a^k_h = a_{(k-1)H+h}$.
We bound each of the three terms separately below.
\paragraph{Bounding $\diff$} $\diff$ is at most $\frac{\epsilon T}{\tau-c^0}$ by using the following lemma directly, which shows the difference in optimal reward when shrinking the constraint threshold by $\epsilon$.
\begin{lemma}
	\label{lem:Jeps}
	For $\epsilon\in[0, \tau-c^0]$, $\optJ - \optJeps \leq \frac{\epsilon}{\tau - c^0}$.
\end{lemma}
%\begin{proof}
%	Since the occupancy measure space is convex, one can find a policy $\pi'$ such that $\mu_{\pi'} = (1-\gamma)\mu_{\optpi} + \gamma\mu_{\pi^0}$ with $\gamma = \frac{\epsilon}{\tau-c^0}$.
%	Now by $J^{\pi,d}=\inner{\mu_{\pi}}{d}$:
%	\begin{align*}
%		J^{\pi', c} &= (1-\gamma)J^{\optpi, c} + \gamma J^{\pi^0, c}\\ 
%		&\leq (1-\gamma)\tau + \gamma c^0 \leq \tau - \epsilon.
%	\end{align*}
%	Thus, $\optJ - \optJeps \leq \optJ - J^{\pi',r} = \gamma (\optJ - J^{\pi^0,r}) \leq \frac{\epsilon}{\tau - c^0}$.
%\end{proof}
%This lemma shows that rate of reward loss when shrinking the constraint threshold is upper bounded by $\frac{1}{\tau-c^0}$.
%The idea of the proof is simple: since the occupancy measure space is convex, and there is a strictly safe policy, one can find a policy whose occupancy measure is a linear combination of $\mu_{\optpi}$ and $\mu_{\pi^0}$, with constraint violation equals to $\tau-\epsilon$.
%The difference $\optJ-\optJeps$ is then related to the coefficient of the linear combination.

\paragraph{Bounding \reg} We decompose the $\reg$ term as
$
	\reg= \sumk (J^{\optpieps, r - \frac{\lambda_k}{\eta}c} - J^{\pi_k, r - \frac{\lambda_k}{\eta}c})
+ \sumk\frac{\lambda_k}{\eta}(J^{\optpieps, c} - J^{\pi_k, c}).
$
The first term can be rewritten by the value difference lemma (\pref{lem:val diff}), and is bounded following the standard OMD analysis.
The final result is shown below.
\begin{lemma}
	\label{lem:po}
	%For any benchmark policy $\piref$ and interval $\calI=\{i,i+1,\ldots,j\}\subseteq[K]$, we have: $\sum_{k\in\calI} (J^{\piref, r - \frac{\lambda_k}{\eta}c} - J^{\pi_k, r - \frac{\lambda_k}{\eta}c}) \leq 32\iota\sqrt{S^2AT\ln\frac{4SAT^3}{\delta}} + 4\iota\sqrt{HT\ln T} + 128S^{1.5}AH\iota\ln^{1.5}(4SAT/\delta)\leq \frac{\lambda}{4(\tau-c^0)}$ with probability at least $1-4\delta$.
	For any policy $\piref$ and a subset of episodes $\calI=\{i,i+1,\ldots,j\}\subseteq[K]$, we have: $\sum_{k\in\calI} J^{\piref, r - \frac{\lambda_k}{\eta}c} - J^{\pi_k, r - \frac{\lambda_k}{\eta}c} %\leq 32\iota\sqrt{S^2AT\ln\frac{4SAT^3}{\delta}} + 4\iota\sqrt{HT\ln T} + 128S^{1.5}AH\iota\ln^{1.5}(4T/\delta)
	\leq \frac{\lambda}{4(\tau-c^0)}$ with probability at least $1-4\delta$.
	%\leq \frac{\lambda NH}{\eta}\sqrt{SK} + \frac{\lambda N}{\eta}\sqrt{S^3AH|\calI|} + \frac{\lambda NHS^2A}{\eta} $.
	%\tilO{ NH\rbr{1 + \frac{\lambda}{\eta}}\sqrt{S^3AK} }$.
\end{lemma}
For the second term, we have
\begin{align*}
	&\sumk\frac{\lambda_k}{\eta}(J^{\optpieps, c} - J^{\pi_k, c}) \overset{\text{(i)}}{\leq} \sumk\frac{\lambda_k}{\eta}\rbr{\tau - \epsilon - J^{\pi_k, c}}\\
	&\overset{\text{(ii)}}{\lesssim} \sumk\frac{\lambda_k}{\eta}\rbr{\tau - \epsilon - \hatJ_k} + \sumk\frac{\lambda_k}{\eta}(\hatJ_k - \E_k[\hatJ_k])\\
	&\overset{\text{(iii)}}{\leq} \frac{1}{\eta}\sumk\lambda_k(\lambda_k-\lambda_{k+1}) + \frac{\tau^2K}{\eta} + \tilO{\frac{\lambda}{\eta}\sqrt{K}}\\
	&\overset{\text{(iv)}}{=} \tilO{\frac{K}{\eta} + \frac{\lambda}{\eta}\sqrt{K} }
	\overset{\text{(v)}}{=} \tilO{\frac{\lambda}{\tau-c^0}}.
\end{align*}
%\begin{align*}
%	&\sumk\frac{\lambda_k}{\eta}\rbr{\tau - \epsilon - J^{\pi_k, c}} \lesssim \sumk\frac{\lambda_k}{\eta}\rbr{\tau - \epsilon - \hatJ_k}\\
%	&\overset{\text{(i)}}{\leq} \frac{1}{\eta}\sumk\lambda_k(\lambda_k-\lambda_{k+1}) + \frac{\tau^2K}{\eta} \overset{\text{(ii)}}{=} \tilO{\frac{K}{\eta}}.
%\end{align*}
Here, (i) is by the definition of $\optpieps$; (ii) is because $\E_k[\hatJ_k]$ is a good estimate of $J^{\pi_k, c}$ ($\E_k$ denotes the expectation given everything before episode $k$);
(iii) applies Azuma's inequality and the following argument: if $\lambda_{k+1}>0$, then $\tau-\epsilon-\hatJ_k \leq \lambda_k-\lambda_{k+1}$ by the definition of $\lambda_k$;
otherwise, $\lambda_k\leq\tau-\epsilon-\hatJ_k<\tau$ and $\lambda_k(\tau-\epsilon-\hatJ_k)\leq\tau^2$;
(iv) is because
$\sumk\lambda_k(\lambda_k-\lambda_{k+1}) = \frac{1}{2}\sumk\rbr{ \lambda_k^2 - \lambda_{k+1}^2 + (\lambda_{k+1} - \lambda_k)^2 } \leq \frac{K}{2}$,
where the last inequality is by $\lambda_1=0$ and $|\lambda_k-\lambda_{k+1}|\leq 1$;
(v) is by the value of the parameters in \pref{eq:parameters}.
Putting everything together, we have $\reg=\tilo{\frac{\lambda}{\tau-c^0}}$.
%\begin{align*}
%	\tilO{  \frac{N}{\tau-c^0}\sqrt{S^2AT} + \frac{N}{\tau-c^0}\sqrt{SHT} + \frac{SAH}{\tau-c^0} }.
%\end{align*}

\paragraph{Bounding \dev} We prove a more general statement saying that $\sumk\sumh(J^{\pi_k,d} - d(s^k_h,a^k_h))\lesssim \lambda$ for any utility function $d\in[0, 1]^{\SA}$; see \pref{lem:stab episode}.
The idea is as follows. 
Using the Bellman equation~\eqref{eq:Bellman},
the left-hand side is equal to $\sumk\sumh(P_{s^k_h, a^k_h}v^{\pi_k,d} - q^{\pi_k,d}(s^k_h, a^k_h))$, which can then be decomposed as the sum of three terms:
$\sumk\sumh(P_{s^k_h, a^k_h}v^{\pi_k,d} -  v^{\pi_k,d}(s^k_{h+1}))$,
$\sumk\sumh(v^{\pi_k,d}(s^k_{h}) -  q^{\pi_k,d}(s^k_h, a^k_h))$,
and 
$\sumk\sumh(v^{\pi_k,d}(s^k_{h+1}) - v^{\pi_k,d}(s^k_{h}))$.
The first two terms are sums of martingale difference sequence and of order $\tilo{\tmix\sqrt{T}}$ by Azuma's inequality.
The last term can be rearranged and telescoped to $\sum_{k=2}^K(v^{\pi_k,d}(s^k_1) - v^{\pi_{k-1},d}(s^k_1))$ (dropping two negligible terms).
Now, this is exactly the term where we need the stability of $\pi_k$: as long as $\pi_k$ changes slowly, this bound is sublinear in $K$.
As discussed, we ensure this by using a new estimator $\hatbeta_k$ whose scale is nicely bounded, allowing us to show the following.
%\begin{align*}
%	&\sumk\sum_{h=1}^H (J^{\pi_k,d} - d(s_t, a_t))\\
%	&\lesssim \sumk\sum_{h=1}^H\rbr{q^{\pi_k,d}(s^k_h, a^k_h) - v^{\pi_k,d}(s^k_h)}\\
%	&+ \sum_{k=2}^K(v^{\pi_k,d}(s^k_1) - v^{\pi_{k-1},d}(s^k_1))\\ 
%	&+ \sumk\sum_{h=1}^H(v^{\pi_k,d}(s^k_{h+1}) - P_{s^k_h, a^k_h}v^{\pi_k,d}).
%\end{align*}
%The first and the last term above are sum of martingale difference sequence and are of order $\tilo{\tmix\sqrt{T}}$ by Azuma's inequality.
%The second term is also of order $\tilo{\sqrt{T}}$ thanks to the stability of policy optimization shown in the following lemma.
\begin{lemma}
	\label{lem:stab}
	For any $k$, we have $\abr{\pi_k(a|s) - \pi_{k-1}(a|s)} \leq 8\theta H\iota\pi_{k-1}(a|s)$ and $\abr{v^{\pi_k,d}(s) - v^{\pi_{k-1},d}(s)} \leq 65\theta HN^2\iota$ where $d\in[0, 1]^{\SA}$ is any utility function.
\end{lemma}
%Note that $\norm{\hatbeta_k}_{\infty}=\tilo{\frac{\lambda N}{\eta}}$ (independent of $T$) is crucial for the stability to hold.

Putting everything together then finishes the proof for $R_T$. %and by the definition of $\epsilon$, we obtain $R_T=\tilo{ \frac{NH}{\tau-c^0}(\sqrt{S^3AT} + \sqrt{SHT} + S^2AH) + \frac{H}{(\tau-c^0)^2} }$.

\paragraph{Constraint Violation}
For $C_T$, we decompose as: $C_T = \sumk\sum_{h=1}^H (c(s^k_h, a^k_h) - J^{\pi_k,c}) + H\sumk(J^{\pi_k,c}-\tau)$.
%\begin{align*}
%	C_T&=\sumt (c(s_t, a_t) - \tau)\\
%	&= \sumk\sum_{h=1}^H (c(s^k_h, a^k_h) - J^{\pi_k,c}) + H\sumk(J^{\pi_k,c}-\tau).
%\end{align*}
The first term is similar to $\dev$ and is of order $\lambda$ by \pref{lem:stab episode}.
The second term is roughly $H\sumk(\hatJ_k-\tau)$ (recall $\hatJ_k$ is a good estimator of $J^{\pi_k,c}$), and can be further bounded by $H(\lambda - K\epsilon)$.
The reason of the last step is that in \pref{lem:bound lambda}, using the interval regret property (\pref{lem:po}) ensured by our OMD update, we show $\lambda_k<\lambda$, that is, the truncation at $\lambda$ never happens in the update rule of $\lambda_k$ (with high probability).
This implies $\lambda_{k+1} \geq \lambda_k + \hatJ_k + \epsilon - \tau$ by definition,
and rearranging thus shows $\sumk (\hatJ_k-\tau) \leq \sumk (\lambda_{k+1} - \lambda_k - \epsilon) = \lambda_{K+1} - K\epsilon \leq \lambda - K\epsilon$.

%For the second term, we have
%\begin{align*}
%	H\sumk(J^{\pi_k,c}-\tau) &\approx H\sumk(\hatJ_k-\tau)\leq H(\lambda - K\epsilon).
%	%&\leq \frac{1}{T} + H\sqrt{2K\ln\frac{4T^3}{\delta}} + H(\lambda - K\epsilon) \tag{\pref{lem:hatJ} and \pref{lem:azuma}}\\
%	%&\leq \tilO{1} + \eta + H(\lambda - K\epsilon).
%\end{align*}
%where the last inequality is by $\lambda_k<\lambda$ (\pref{lem:bound lambda}) and the definition of $\lambda_k$.
%We note that the proof of \pref{lem:bound lambda} requires the interval regret property (\pref{lem:po}).
Now if $\epsilon=3\lambda/K$ (\emph{c.f.} \pref{eq:parameters}), the negative term $-HK\epsilon$ above cancels out all the positive terms and $C_T\leq 0$.
Otherwise, we have $T=\tilo{\frac{N^2H^2S^3A + N^2H^3S}{(\tau-c^0)^2} + \frac{NH^2S^2A}{\tau-c^0}}$ by the definition of $\epsilon$, and the trivial bound $C_T\leq T$ concludes the proof after we plug in the definition of $N$ and $H$.

%% file: weak.tex
% !TEX root = main.tex

In this section, we drop the ergodic assumption and consider the most general case of weakly communicating MDPs.
As in the unconstrained case, the span of the bias function of the optimal policy $\optpi$ plays an important role in this case and is unavoidable in the regret bound. 
More concretely, our bounds depend on $\sp(v^{\optpi,r})$ and $\sp(v^{\optpi,c})$, and our algorithm assumes knowledge of these two parameters, which we also write as $\sp^{\star}_r$ and $\sp^{\star}_c$ for short.
We note that even in the unconstrained case, all existing algorithms whose regret bound depends on $\sp^{\star}_r$ need the knowledge of $\sp^{\star}_r$.

%In this section, we study the more general weakly communicating assumption.
%An MDP is weakly communicating if its state space consists of two subsets:
%in the first subset, all states are transient under any stationary policy;
%in the second subset, every two state are communicating under some policy.
%We assume that $\sp(v^{\optpi,r})$ and $\sp(v^{\optpi,c})$ are known to the learner.

Weakly communicating MDPs impose extra challenges in learning.
Specifically, there is no uniform bound for $\sp(v^{\pi,r})$ and $\sp(v^{\pi,c})$ for all stationary policy $\pi$ (while in the ergodic case they are both $\tilo{\tmix}$).
It is also unclear how to obtain an accurate estimate of a policy's bias function as in ergodic MDPs, which as we have shown is an important step for policy optimization algorithm.
We are thus not able to extend the approach from \pref{sec:ergodic} to this general case.
Instead, we propose to solve the problem via a finite-horizon approximation, which is in spirit similar to another algorithm of~\citep{wei2020model} (called Optimistic $Q$-Learning).

%we lose a uniform bound on $\sp(v^{\pi,r})$ and $\sp(v^{\pi,c})$ for any stationary policy $\pi$, which are of order $\tilo{\tmix}$ under the ergodic assumption.
%We also cannot easily obtain an accurate estimate of any policy's value function as in ergodic MDPs, which is crucial for the policy evaluation step in PO.
%Therefor, PO approach developed in the last section is no longer applicable.

Specifically, we still divide the $T$ steps into $K$ episodes, each of length $H$.
In each episode, we treat it as an episodic finite-horizon MDP, and try to find a good (non-stationary) policy through the lens of occupancy measure, in which expected reward and cost are both linear functions and easy to optimize over.
Concretely, consider a fixed starting state $s$, a non-stationary policy $\pi\in(\Delta_{\calA})^{\calS\times[H]}$ whose behavior can vary in different steps within an episode, and an inhomogeneous transition function $P=\{P_{s, a, h}\}_{(s, a, h)\in\SA\times[H]}$ where $P_{s, a, h}\in \Delta_\calS$ specifies the probability of next state after taking action $a$ at state $s$ and step $h$.
The corresponding occupancy measure $\nu_{\pi,P,s}\in[0, 1]^{\SA\times[H]\times\calS}$ is then such that $\nu_{\pi,P,s}(s', a, h, s'')$ is the probability of visiting state $s'$ at step $h$, taking action $a$, and then transiting to state $s''$, if the learner starts from state $s$, executes $\pi$ for the next $H$ steps, and the transition dynamic follows $P$. %\footnote{Although the true transition $P$ is always homogeneous (that is, same across different steps of an episode), our estimated transition computed implicitly from some optimization problem can be inhomogeneous.}

%To this end, we resort to solve the original problem via learning a finite-horizon MDP.
%Specifically, for the given MDP $\calM$, we define an inhomogeneous finite-horizon MDP $\tilcalM=(\calS, \calA, \tilP, H)$ where $H$ is the horizon, $\tilP=\{P_{s, a, h}\}_{(s, a, h)\in\SA\times[H]}$, $P_{s, a, h}=P_{s, a}$, and $P_{s, a, h}(s')$ is the probability of transiting from state $s$ in layer $h$ to state $s'$ in layer $h+1$ by taking action $a$.

%To introduce the algorithm, we also need the concept of occupancy measure for $\tilcalM$.
%For any policy $\pi\in(\Delta_{\calS})^{\SA\times[H]}$ (when following $\pi$ in $\tilcalM$, the probability of taking action $a$ at state $s$ in layer $h$ is $\pi(a|s, h)$), and transition function $P=\{P_{s,a,h}\}_{s,a,h}$, the associated occupancy measure $\nu_{\pi,P,s}\in[0, 1]^{\SA\times[H]\times\calS}$ satisfies
%\begin{align*}
%	&\nu_{\pi,P,s}(s', a, h, s'')\\
%	&= P(s_h=s', a_h=a, s_{h+1}=s''|\pi,P, s_1=s),
%\end{align*}
%that is, $\nu_{\pi,P}(s', a, h, s'')$ is the probability of visiting state $s'$ and taking action $a$ in layer $h$, and then transiting to state $s''$ in layer $h+1$, when following policy $\pi$ in a finite-horizon MDP with transition $P$ and initial state $s$.

Conversely, a function $\nu \in[0, 1]^{\SA\times[H]\times\calS}$ is an occupancy measure with respect to a starting state $s$, some policy $\pi_\nu$, and transition $P_\nu$ if and only if it satisfies:
%To identify all possible occupancy measures of $\tilcalM$, we also define the set $\calV_s$ for any $s\in\calS$, such that each element $\nu\in\calV_s$ satisfies the following three properties:
\begin{enumerate}
	\item Initial state is $s$: $\sum_a\sum_{s''}\nu(s', a, 1, s'')=\Ind\{s'=s\}$.
	\item Total mass is $1$: $\sum_{s'}\sum_a\sum_{s''}\nu(s', a, h, s'')=1, \;\forall h$.
	\item Flow conservation: $\sum_{s''}\sum_{a}\nu(s'', a, h, s')=\sum_{a}\sum_{s''}\nu(s', a, h+1, s'')$ for all $s'\in\calS, h\in[H-1]$.
\end{enumerate}
We denote the set of all such $\nu$ as $\calV_s$.
For notational convenience, for a $\nu \in \calV_s$, we define $\nu(s', a, h)=\sum_{s''}\nu(s', a, h, s'')$, $\nu(s', a)=\sum_h\nu(s', a, h)$, and $\nu(s', h)=\sum_a\nu(s', a, h)$.
Also note that the corresponding policy $\pi_\nu$ and transition $P_\nu$ can be extracted via
$\pi_{\nu}(a|s', h)=\frac{\nu(s', a, h)}{\nu(s', h)}$ and $P_{\nu, s', a, h}(s'')=\frac{\nu(s', a, h, s'')}{\nu(s', a, h)}$.
These facts are taken directly from~\citep{rosenberg2019online} (although there the starting state is always fixed).
Note that $\calV_s$ is a convex polytope with polynomial constraints.
Also note that if one were to enforce $P_\nu$ to be homogeneous (that is, same across different steps of an episode), $\calV_s$ would become non-convex.
This is why we consider inhomogeneous transitions even though we know that the true transition is indeed homogeneous.
%By \citep[Lemma 1]{jin2019learning}, there is a one-to-one correspondence between occupancy measures with initial state $s$ and elements in $\calV_s$: clearly $\nu_{\pi,P,s}\in\calV_s$ for any policy $\pi$ and transition $P$.
%Conversely, each element $\nu\in\calV_s$ is the occupancy measure w.r.t policy $\pi_{\nu}$, transition $P_{\nu}$ and initial state $s$.

\DontPrintSemicolon 
\setcounter{AlgoLine}{0}
\begin{algorithm}[t]
	\caption{Finite Horizon Approximation for CMDP}
	\label{alg:weak}
	\textbf{Define:} $H=\lceil(T/S^2A)^{1/3} \rceil$, $K=T/H$.
	
	\For{$k=1,\ldots,K$}{
		Observe current state $s^k_1 = s_{(k-1)H+1}$.
	
		Compute occupancy measure:
		\begin{equation}\label{eq:OPT1}
		\nu_k=\argmax_{\nu\in\calV_{k, s^k_1}: \inner{\nu}{c}\leq H\tau + \sp^\star_c}\inner{\nu}{r},
		\end{equation}
		where $\calV_{k, s}=\{\nu\in\calV_s: P_{\nu}\in\calP_k\}$ (see \pref{eq:conf}). %\label{line:om}
		
		Extract policy $\pi_k=\pi_{\nu_k}$ from $\nu_k$.
		
		\For{$h=1,\ldots,H$}{
			Play action $a^k_h\sim\pi_k(\cdot|s^k_h, h)$ and transit to $s^k_{h+1}$.%\label{line:execute}
		}
		%Execute $\pi_k=\pi_{\nu_k}$ for $H$ steps.
	}
	
\end{algorithm}

With the help of occupancy measure, we present two algorithms below.
As far as we know, these are the first provable results for general weakly communicating CMDP.

\subsection{An Efficient Algorithm}
Our first algorithm is simple and computationally efficient;
see \pref{alg:weak} for the pseudocode.
At the beginning of each episode $k$, our algorithm observes the current state $s^k_1$ and finds an occupancy measure $\nu_k\in \calV_{k, s^k_1}$ that maximizes expected reward $\inner{\nu}{r} = \sum_{s,a} \nu(s,a) r(s,a)$ under the cost constraint  $\inner{\nu}{c}\leq H\tau + \sp^\star_c$.
Here, $\calV_{k, s^k_1}$ is a subset of $\calV_{s^k_1}$ such that $P_{\nu}$ for each $\nu \in \calV_{k, s^k_1}$ lies in a standard Bernstein-type confidence set $\calP_k$ defined as
%We are now ready to introduce our algorithm (\pref{alg:weak}).
%We divide all time steps into $K=T/H$ episodes, where each episode corresponds to an episode in $\tilcalM$.
%In episode $k\in[K]$, we solve an optimistic occupancy measure $\nu_k$ of $\tilcalM$ within some confidence set (\pref{line:om}).
%Specifically, $\nu_k$ maximizes the expected reward while ensuring that the total cost is upper bounded by $H\tau+\sp(v^{\optpi,c})$, and the induced transition lies in a standard Bernstein type confidence set $\calP_k$ defined as
\begin{align}
	\calP_k &= \Big\{ P'=\{P'_{s, a,h}\}_{(s, a, h)\in\SA\times[H]}, P'_{s, a,h}\in\Delta_{\calS}:\notag\\
	&\abr{P'_{s, a, h}(s') - \P_{k, s, a}(s')}\notag\\
	&\leq 4\sqrt{\P_{k,s,a}(s')\alpha_k(s, a) } + 28\alpha_k(s, a) \Big\},\label{eq:conf}
\end{align}
where $\P_k$ is the same empirical transition function defined in \pref{eq:empirical_transition}, $\alpha_k(s, a)=\frac{\iota'}{\Np_k(s, a)}$, %$\Np_k(s,a)=\max\{1,\N_k(s, a)\}$, $\N_k(s, a)$ is the number of visits to $(s, a)$ before episode $k$, 
and $\iota'=\ln\frac{2SAT}{\delta}$. %$\P_{k, s, a}(s')=\frac{\N_k(s, a, s')}{\Np_k(s, a)}$ is the empirical transition, $\N_k(s, a, s')$ is the number of visits to triplet $(s, a, s')$ before episode $k$, and $\delta>0$ is some failure probability.
As a standard practice, $\calP_k$ is constructed in a way such that it contains the true transition with high probability (\pref{lem:conf}).
With $\nu_k$, we simply follow the policy $\pi_k=\pi_{\nu_k}$ extracted from $\nu_k$ for the next $H$ steps.

%\begin{lemma}
%	\label{lem:conf}
%	With probability at least $1-\delta$, $\tilP\in\calP_k$,$\forall k$.
%\end{lemma}
%, whose definition is deferred to \pref{app:calP}.
%We also abuse the notation to define $c(s, a, h)=c(s, a)$ and $r(s, a, h)=r(s, a)$ for any $(s, a, h)\in\SA\times[H]$, so that the dot product in \pref{line:om} is well-defined.
%Here, $\calV_s\subseteq[0, 1]^{\SA\times[H]\times\calS}$ is the set of all possible occupancy measures of the finite-horizon MDP with initial state $s$, that is, 

Note that the key optimization problem~\eqref{eq:OPT1} in this algorithm can be efficiently solved, because the objective function is linear and the domain is again a convex polytope with polynomial constraints, thanks to the use of occupancy measures.
We now state the main guarantee of \pref{alg:weak}.
\begin{theorem}
	\label{thm:weak}
	\pref{alg:weak} ensures with probability at least $1-10\delta$,
	\begin{align*}
		R_T &= \tilO{ (1+\sp^\star_r)(S^2A)^{1/3}T^{2/3} },\\
		C_T &= \tilO{ (1+\sp^\star_c)(S^2A)^{1/3}T^{2/3} }.
	\end{align*}
	%$R_T= \tilO{\sp(v^{\optpi, r})T^{2/3} + S\sqrt{A}T^{2/3} + S^2AT^{1/3} }$ and $C_T=\tilO{\sp(v^{\optpi,c})T^{2/3} + S\sqrt{A}T^{2/3} + S^2AT^{1/3} }$.
\end{theorem}

%\subsection{\pfref{thm:weak}}
\paragraph{Analysis}
As the first step, we need to quantify the finite-horizon approximation error.
For any non-stationary policy $\pi\in(\Delta_{\calA})^{\calS\times[H]}$, inhomogeneous transition function $P=\{P_{s, a, h}\}_{(s, a, h)}$, and utility function $d$, define value function $V^{\pi,P,d}_h(s)=\E\big[\sum_{h'=h}^Hd(s_{h'}, a_{h'})|\pi, P, s_h=s\big]$ where $a_{h'} \sim \pi(\cdot|s_{h'}, h')$ and $s_{h'+1} \sim P_{s_{h'}, a_{h'}, h'}$, and similarly action-value function $Q^{\pi,P,d}_h(s, a)=\E\big[\sum_{h'=h}^Hd(s_{h'}, a_{h'})|\pi, P, s_h=s, a_h=a\big]$.
Additionally, define $V^{\pi,P,d}_{H+1}(s)=Q^{\pi,P,d}_{H+1}(s, a)=0$.
Further let $\tilP=\{\tilP_{s, a, h}\}_{(s, a, h)}$ be such that $\tilP_{s,a,h} = P_{s,a}$ (the true transition function) for all $h$.
We often ignore the dependency on $\tilP$ and $r$ for simplicity when there is no confusion.
For example, $V^{\pi}_h$ denotes $V^{\pi,\tilP,r}_h$ and $V^{\pi,c}_h$ denotes $V^{\pi,\tilP,c}_h$.
For a stationary policy $\pi\in(\Delta_{\calA})^{\calS}$, define $\tilpi$ as the policy that mimics $\pi$ in the finite-horizon setting, that is, $\tilpi(\cdot|s, h)=\pi(\cdot|s)$.
The following lemma shows the approximation error.
%An important step in establishing the regret bound is to quantify the approximation error between the finite-horizon MDP and the original MDP, which are shown in the following lemma.
\begin{lemma}
	\label{lem:VJ}
	%For any stationary policy $\pi\in(\Delta_{\calA})^{\calS}$, define $\tilpi\in(\Delta_{\calA})^{\calS\times[H]}$ such that $\tilpi(\cdot|s,h)=\pi(\cdot|s)$.
	For any stationary policy $\pi\in(\Delta_{\calA})^{\calS}$ and utility function $d\in\fR^{\SA}$ such that $J^{\pi,d}(s)=J^{\pi,d}$ for all $s\in\calS$, we have $|V^{\tilpi,d}_h(s) - (H-h+1)J^{\pi,d}|\leq \sp(v^{\pi,d})$ for all state $s\in\calS$ and $h\in[H]$.
\end{lemma}

Next, in \pref{lem:opt} (see appendix), we show that because $\calP_k$ contains $\tilP$ with high probability, the occupancy measure with respect to the policy $\tiloptpi$ and the transition $\tilP$ is in the domain of~\eqref{eq:OPT1} as well.
%
%We have the following result.
%%Also define $\optnu_s=\nu_{\tiloptpi, P, s}$.
%\begin{lemma}
%	\label{lem:opt}
%	Under the event of \pref{lem:conf}, \pref{alg:weak} ensuers $\nu_{\tiloptpi,\tilP,s^k_1}$ lies in the decision set of $\nu_k$.
%\end{lemma}
%\begin{proof}
%	By \pref{lem:VJ}, we have:
%	\begin{align}
%		\inner{\nu_{\tiloptpi,\tilP,s^k_1}}{c} &= V^{\tiloptpi, c}_1(s^k_1) \leq \sp(v^{\optpi, c}) + HJ^{\optpi, c}\notag\\
%		&\leq \sp(v^{\optpi, c}) + H\tau.\label{eq:opt}
%	\end{align}
%	Then by $\tilP\in\calP_k$, the statement is proved.
%\end{proof}
%
As the last preliminary step, we bound the bias in value function caused by transition estimation, that is, difference between $\tilP$ and $P_k = P_{\nu_k}$.
\begin{lemma}
	\label{lem:V diff}
	For any utility function $d\in[0, 1]^{\SA}$, with probability at least $1-4\delta$, we have $|\sumk (V^{\pi_k, d}_1(s^k_1) - V^{\pi_k,P_k,d}_1(s^k_1))| = \tilo{\sqrt{S^2AH^2K} + H^2S^2A}$.
\end{lemma}

We are now ready to prove \pref{thm:weak}.
\begin{proof}[\pfref{thm:weak}]
	We decompose $R_T$ into three terms:
	\begin{align*}
		&R_T = \sumt \optJ - r(s_t, a_t) = \sumk\rbr{H\optJ - \sumh r(s^k_h, a^k_h)}\\ 
		&= \sumk\rbr{H\optJ - V^{\tiloptpi}_1(s^k_1)} + \sumk\rbr{V^{\tiloptpi}_1(s^k_1) - V^{\pi_k}_1(s^k_1)}\\ 
		&\qquad + \sumk\rbr{V^{\pi_k}_1(s^k_1) - \sumh r(s^k_h, a^k_h)}.
		%&\leq K\sp(v^{\optpi, r}) + \sumk\rbr{V^{\tiloptpi}_1(s^k_1) - V^{\pi_k}_1(s^k_1)} + \tilO{H\sqrt{K}}. \tag{\pref{lem:VJ} and \pref{lem:azuma}}
	\end{align*}
	The first term above is upper bounded by $K\sp^\star_r$ by \pref{lem:VJ}.
	For the second term, by \pref{lem:opt} and \pref{lem:V diff}, we have
	\begin{align*}
		&\sumk\rbr{V^{\tiloptpi}_1(s^k_1) - V^{\pi_k}_1(s^k_1)}
		\leq \sumk\rbr{V^{\pi_k, P_k}_1(s^k_1) - V^{\pi_k}_1(s^k_1)}\\
		&= \tilO{\sqrt{S^2AH^2K} + H^2S^2A}.
	\end{align*}
	The last term is of order $\tilo{H\sqrt{K}}$ by Azuma's inequality (\pref{lem:azuma}).
	Using the definition of $H$ and $K$, we arrive at
	\begin{align*}
		R_T = \tilO{ (1+\sp^\star_r))(S^2A)^{1/3}T^{2/3}}.
	\end{align*}
	For constraint violations, we decompose $C_T$ as:
	\begin{align*}
		&\sumt (c(s_t, a_t) - \tau)
		= \sumk\rbr{\sumh c(s^k_h, a^k_h) - V^{\pi_k, c}_1(s^k_1)}\\
		&\qquad+ \sumk\rbr{ V^{\pi_k, c}_1(s^k_1) - V^{\pi_k, P_k, c}_1(s^k_1) }  \\
		&\qquad + \sumk\rbr{ V^{\pi_k, P_k, c}_1(s^k_1) - H\tau }.
		%&= \tilO{H\sqrt{K} + \sqrt{S^2AH^2K} + H^{1.5}S^2A } \tag{\pref{lem:azuma}, \pref{lem:V diff}, and definition of $P_k,\pi_k$} + K\sp(v^{\optpi,c})\\
		%&= \tilO{ (1+\sp(v^{\optpi,c}))(S^2A)^{1/3}T^{2/3} + S\sqrt{AT}}.
	\end{align*}
	The first term is again of order $\tilo{H\sqrt{K}}$ by Azuma's inequality.
	The second term is of order $\tilO{\sqrt{S^2AH^2K} + H^2S^2A}$ by \pref{lem:V diff}.
	The third term is upper bounded by $K\sp^\star_c$ because of the constraint $\inner{\nu}{c}\leq H\tau + \sp^\star_c$ in the optimization problem~\eqref{eq:OPT1}.
	Using the definition of $H$ and $K$, we get:
	\begin{align*}
		C_T = \tilO{ (1+\sp^\star_c)(S^2A)^{1/3}T^{2/3} }.
	\end{align*}
	This completes the proof.
\end{proof}

\subsection{An Improved but Inefficient Algorithm}

%\pref{thm:weak} only gives suboptimal bounds in $T$.
The bottleneck of the last algorithm/analysis is that the span of value functions are bounded by $H$, which is $T$-dependent and leads to sub-optimal dependency on $T$ eventually.
Below, we present an inefficient variant that achieves $\tilo{\sqrt{T}}$ bounds for both regret and constraint violation.

The only new ingredient compared to \pref{alg:weak} is to enforce a proper upper bound on the span of value functions.
Specifically, for any occupancy measure $\nu$ and utility function $d$, define $V^{\nu,d}_h=V_h^{\pi_{\nu},P_{\nu},d}$.
We then enforce constraints $\sp(V^{\nu_k,r}_h)\leq 2\sp^\star_r$ and $\sp(V^{\nu_k,c}_h)\leq 2\sp^\star_c$; see the new domain $\calW_{k,s}$ in the optimization problem~\eqref{eq:OPT2} of \pref{alg:weak sp}. %\footnote{Note that when $\nu_k(s,a)=0$ for some $(s, a)$, the policy and transition function at $(s, a)$ is undetermined, and we need to search over the valid ones w.r.t the span constraints.}
%We summarize the ideas above in \pref{alg:weak sp}, and its guarantee is stated below.
This new domain is generally non-convex, making it unclear how to efficiently solve this optimization problem. 

\DontPrintSemicolon 
\setcounter{AlgoLine}{0}
\begin{algorithm}[t]
	\caption{Finite Horizon Approximation for CMDP with Span Constraints}
	\label{alg:weak sp}
	\textbf{Define:} $H=\sqrt{T/S^2A}$, $K=T/H$.
	
	\For{$k=1,\ldots,K$}{
		Observe current state $s^k_1 = s_{(k-1)H+1}$.
	
		Compute occupancy measure:
		\begin{equation}\label{eq:OPT2}
		\nu_k=\argmax_{\nu\in\calW_{k, s^k_1}: \inner{\nu}{c}\leq H\tau + \sp^\star_c}\inner{\nu}{r},
		\end{equation}
		where $\calW_{k, s}=\big\{\nu\in\calV_s: P_{\nu}\in\calP_k \; \text{ and } \; \forall h\in[H],   \sp(V^{\nu, r}_h) \leq 2\sp^\star_r, \sp(V^{\nu, c}_h) \leq 2\sp^\star_c \big \}$.\label{line:om span}
		
		Extract policy $\pi_k=\pi_{\nu_k}$ from $\nu_k$.
		
		\For{$h=1,\ldots,H$}{
			Play action $a^k_h\sim\pi_k(\cdot|s^k_h, h)$ and transit to $s^k_{h+1}$.%\label{line:execute}
		}
	}
	
\end{algorithm}

Nevertheless, we show the following improved guarantees.
\begin{theorem}
	\label{thm:weak sp}
	\pref{alg:weak sp} ensures with probability at least $1-6\delta$, $R_T = \tilo{ \sp^\star_r S\sqrt{AT} }$ and $C_T = \tilo{ \sp^\star_c S\sqrt{AT} }$.
\end{theorem}

\paragraph{Analysis}
Similarly to proving \pref{thm:weak}, we first show in \pref{lem:opt sp} that the occupancy measure with respect to the policy $\tiloptpi$ and the transition $\tilP$ is in the domain of~\eqref{eq:OPT2}.
%\begin{lemma}
%	\label{lem:opt sp}
%	Under the event of \pref{lem:conf}, \pref{alg:weak sp} ensures $\nu_{\tiloptpi,P,s^k_1}$ lies in the decision set of $\nu_k$.
%\end{lemma}
%\begin{proof}
%	Note that \pref{eq:opt} still holds.
%	Moreover, by \pref{lem:VJ}, for any two states $s, s'$ and $h\in[H]$:
%	\begin{align*}
%		&|V^{\tiloptpi}_h(s)-V^{\tiloptpi}_h(s')| \leq |V^{\tiloptpi}_h(s) - (H-h+1)J^{\optpi,r}|\\
%		&\qquad + |V^{\tiloptpi}_h(s') - (H-h+1)J^{\optpi,r}| \leq 2\sp(v^{\optpi,r}),\\
%		&|V^{\tiloptpi,c}_h(s)-V^{\tiloptpi,c}_h(s')| \leq |V^{\tiloptpi,c}_h(s) - (H-h+1)J^{\optpi,c}|\\ 
%		&\qquad + |V^{\tiloptpi,c}_h(s') - (H-h+1)J^{\optpi,c}| \leq 2\sp(v^{\optpi,c}).
%	\end{align*}
%	Then by $\tilP\in\calP_k$, the statement is proved.
%\end{proof}
%
%For simplicity of analysis, we adopt a slightly different decomposition on $R_T$ and $C_T$, which mainly uses the following lemma.
We will also need the following key lemma.
\begin{lemma}
	\label{lem:V-d}
	For some utility function $d\in[0, 1]^{\SA}$, suppose $\sp(V^{\pi_k,P_k,d}_h)\leq B$ for all episodes $k\in[K]$ and $h\in[H]$.
	Then, with probability at least $1-2\delta$,
	\begin{align*}
		&\abr{\sumk \rbr{V^{\pi_k, P_k, d}_1(s^k_1) - \sumh d(s^k_h, a^k_h)} }\\
		&= \tilO{(B+1)S\sqrt{AT} + BHS^2A}.
	\end{align*}
\end{lemma}
%We are now ready to prove \pref{thm:weak sp}.
\begin{proof}[\pfref{thm:weak sp}]
	We decompose the regret as follows:
	\begin{align*}
		R_T &= \sumt \optJ - r(s_t, a_t) = \sumk\rbr{H\optJ - V^{\tiloptpi}_1(s^k_1)}\\ 
		&\qquad + \sumk\rbr{V^{\tiloptpi}_1(s^k_1) - \sumh r(s^k_h, a^k_h)}\\
		&\leq K\sp^\star_r + \sumk\rbr{V^{\pi_k, P_k}_1(s^k_1) - \sumh r(s^k_h, a^k_h)}. \tag{\pref{lem:VJ} and \pref{lem:opt sp}}
	\end{align*}
	By \pref{lem:V-d} with $B=2\sp^\star_r$ and the definition of $H$ and $K$, we have shown
%	\begin{align*}
		$R_T = \tilO{(\sp^\star_r+1)S\sqrt{AT}}$.
		%\tilO{\sp(v^{\optpi, r})\sqrt{SAT} + \sp(v^{\optpi, r})(\sqrt{S^2AT} + S^{1.5}AH) } = \tilO{ \sp(v^{\optpi, r})S\sqrt{AT} }.
%	\end{align*}

	For constraint violation, by \pref{lem:V-d} with $B=2\sp^\star_c$, and $V^{\pi_k, P_k, c}_1(s^k_1) - H\tau\leq \sp^\star_c$ due to the constraint $\inner{\nu}{c}\leq H\tau + \sp^\star_c$ in~\eqref{eq:OPT2}, we have
	\begin{align*}
		C_T &= \sumt (c(s_t, a_t) - \tau) = \sumk\rbr{ V^{\pi_k, P_k, c}_1(s^k_1) - H\tau }\\
		&+ \sumk\rbr{\sumh c(s^k_h, a^k_h) - V^{\pi_k, P_k, c}_1(s^k_1)}\\ 
		&= \tilO{ ( \sp^\star_c+1)S\sqrt{AT} }.
	\end{align*}
	This completes the proof.
\end{proof}

We leave the question of how to achieve the same $\tilo{\sqrt{T}}$ results with an efficient algorithm as a key future direction.

%\paragraph{Constant Constraint Violation} Similar to the ergodic case, we can slightly shrink the constraint threshold used in the algorithm to have more conservative policies, which help to get constant constraint violation.
%A caveat here is that we no longer have a uniform upper bound on the span of the bias function w.r.t any policy $\pi$ (in the ergodic case, we have $\sp(v^{\pi,r})=\bigo{\tmix}$).
%Thus, we need the knowledge of a uniform upper bound on the optimal policies w.r.t perturbed constraint thresholds.
%Specifically, we need the knowledge of $\max_{\epsilon\in[0, \tau-c^0]}\sp(v^{\optpieps,c})$.
%In \pref{app:cc}, we give a variant of \pref{alg:weak} that ensures $\tilo{T^{2/3}}$ regret and constant constraint violation.

%% file: app-pre.tex
% !TEX root = main.tex

\paragraph{Notations} Note that all algorithms proposed in this paper divide $T$ steps into $K$ episodes.
Throughout the appendix, denote by $\E_k[\cdot]$ the expectation conditioned on the events before episode $k$ and define $P^k_h=P_{s^k_h, a^k_h}$.

\paragraph{Issues of Some Related Work}
In \citep{agarwal2021concave}, they bound the span of the bias function w.r.t learner's policy and some estimated transition function by diameter $D$; see their Equation (90).
They directly quote \citep{jaksch2010near} as the reasonings.
However, the arguments in \citep{jaksch2010near} is only applicable when the learner's policy is computed by Extended Value Iteration without constraints, while their learner's policy is computed by solving some constrained optimization problem.

In \citep[Lemma 11]{agarwal2021markov}, they claim that the span of the bias function of any stationary policy is upper bounded by $D$, which is clearly not true.
Again, their learner's policy is computed by solving some constrained optimization problem.

%In \citep{agarwal2021markov,agarwal2021concave}, they also divide the total number of steps into episodes.
%In episode $k$, they compute policy $\pi_k$ by solving a 
%Then in their analysis, they bound $\sp(v^{\pi_k,P_k,r})$ by diameter $D$ directly (see \citep[Lemma 11]{agarwal2021markov} and Equation (90) in \citep{agarwal2021concave}).

%% file: app-ergodic.tex
% !TEX root = main.tex

\paragraph{Notations} 
%We write $\mu_{\pi,P}$ as $\mu_{\pi}$ and define $\mu_{\pi}(s)=\sum_a\mu_{\pi}(s, a)$.
Define function $\hatV_k$ such that $\hatbeta_k(s, a) = r(s, a) - \frac{\lambda_k}{\eta}c(s, a) + \P_{k,s,a}\hatV_k$.
Note that $\norm{\hatV_k}_{\infty}\leq \frac{2\lambda N}{\eta}$.
For any subset of episodes $\calI=\{i,\ldots,i+1,j\}\subseteq[K]$, define $\calI_{[1]}=i$ as the smallest element in $\calI$.

\subsection{\pfref{lem:Jeps}}
\begin{proof}
	Since the occupancy measure space is convex, one can find a policy $\pi'$ such that $\mu_{\pi'} = (1-\gamma)\mu_{\optpi} + \gamma\mu_{\pi^0}$ with $\gamma = \frac{\epsilon}{\tau-c^0}$.
	Now by $J^{\pi,d}=\inner{\mu_{\pi}}{d}$:
	\begin{align*}
		J^{\pi', c} &= (1-\gamma)J^{\optpi, c} + \gamma J^{\pi^0, c} \leq (1-\gamma)\tau + \gamma c^0 \leq \tau - \epsilon.
	\end{align*}
	Thus, $\optJ - \optJeps \leq \optJ - J^{\pi',r} = \gamma (\optJ - J^{\pi^0,r}) \leq \frac{\epsilon}{\tau - c^0}$.
\end{proof}

\subsection{Bounding Estimation Error of $\hatbeta_k$ and $\hatJ_k$}
The next two lemmas bound the bias of $\hatbeta_k$ and $\hatJ_k$ w.r.t the quantities they estimate.

\begin{lemma}
	\label{lem:hatJ}
	$\abr{\E_k[\hatJ_k] - J^{\pi_k, c}} \leq 1/T^2$.
\end{lemma}
\begin{proof}
	%Denote by $s^k_i$ the $i$-th time step in episode $k$.
	For $h \geq N$, we have $\norm{(P^{\pi_k})^h_{s, \cdot} - \mu_{\pi_k}}_1 \leq \frac{2}{T^4}$ for any $s\in\calS$ by \pref{lem:mix}.
	Thus,
	\begin{align*}
		\abr{\E_k[\hatJ_k] - J^{\pi_k, c}} &= \abr{\frac{1}{H - N}\sum_{h=N}^{H-1}\sum_{s'}((P^{\pi_k})^h_{s^k_1, s'} - \mu_{\pi_k}(s'))\sum_a\pi_k(a|s')c(s', a)} \leq 1/T^2.
	\end{align*}
\end{proof}

\begin{lemma}
	\label{lem:hatVk}
	$\abr{\E_k[\hatV_k(s)] - v^{\pi_k, r - \frac{\lambda_k}{\eta}c}(s) - NJ^{\pi_k, r - \frac{\lambda_k}{\eta}c}} \leq \frac{\lambda}{\eta T}$.
\end{lemma}
\begin{proof}
	Note that $\hatV_k(s)=\sum_a\pi_k(a|s)\hatQ_k(s, a)$ where $\hatQ_k$ is the estimator $\hatbeta_k$ in \citep[Lemma 6]{wei2020model} with reward $r-\frac{\lambda_k}{\eta}c$, and we have $|\E_k[\hatQ_k(s, a)] - q^{\pi_k,r-\frac{\lambda_k}{\eta}c}(s, a) - NJ^{\pi_k,r-\frac{\lambda_k}{\eta}c}| \leq \frac{30\lambda}{\eta T^2}$ (the constant is determined by tracing their proof of Lemma 6 in Appendix B.2)  by $\norm{r-\frac{\lambda_k}{\eta}c}_{\infty}\leq2\lambda/\eta$.
	Then by $T\geq 30\max\{\tmix,\thit\}$ and $\sum_a\pi_k(a|s)q^{\pi_k,r-\frac{\lambda_k}{\eta}c}(s, a)=v^{\pi_k,r-\frac{\lambda_k}{\eta}c}(s)$, the statement is proved.
%	Fixed state $s$.
%	Note that $\hatV_k(s)=\frac{1}{M}\sum_{i=1}^MR_{\tau_i}$, where $M$ is the total number of intervals and $\tau_i$ is the time step of the $i$-th interval.
%	Also define $y=r - \frac{\lambda_k}{\eta}c$ and $y^{\pi} = r^{\pi} - \frac{\lambda_k}{\eta}c^{\pi}$.
%	Then,
%	\begin{align*}
%		\E[R_{\tau_i}| s_{\tau_i}=s] &= \sum_{j=0}^{N-1}e_s^{\top}(P^{\pi_k})^jy^{\pi_k} = \sum_{j=0}^{N-1}(e_s^{\top}(P^{\pi_k})^j - \mu_{\pi_k}^{\top})y^{\pi_k} + NJ^{\pi_k, y}\\
%		&= v^{\pi_k, y}(s) + NJ^{\pi_k, y} - \sum_{j=N}^{\infty}(e_s^{\top}(P^{\pi_k})^j - \mu_{\pi_k}^{\top})y^{\pi_k}.
%	\end{align*}
%	By the value of $N$, $\abr{\sum_{j=N}^{\infty}(e_s^{\top}(P^{\pi_k})^j - \mu_{\pi_k}^{\top})y^{\pi_k}}\leq \frac{1}{T^3}$.
%	Then we bound $\E_k[\hatV_k(s)]$...
\end{proof}

\subsection{\pfref{lem:stab}}
\begin{proof}
	By the update rule of $\pi_k$ and following the proof of \citep[Lemma 17]{chen2021impossible} (by \pref{lem:Q span}, we have $c_{\max}=2H\iota$ in their proof), we have 
	$$\pi_k(a|s)\in \sbr{\exp\rbr{-4\theta H\iota}, \exp\rbr{4\theta H\iota}}\pi_{k-1}(a|s).$$
	%$$\pi_k(a|s)\in \sbr{\exp\rbr{-\theta \rbr{1 + \frac{\lambda_k}{\eta}}NH\sqrt{S}}, \exp\rbr{\theta \rbr{1 + \frac{\lambda_k}{\eta}}NH\sqrt{S}}}\pi_{k-1}(a|s).$$
	Therefore, by $|e^x-1| \leq 2|x|$ for $x\in[-1, 1]$, we have $\abr{\pi_k(a|s)-\pi_{k-1}(a|s)} \leq 8\theta H\iota\pi_{k-1}(a|s)$.
	For the second statement, first note that by \pref{lem:val diff} and \pref{lem:bound v}:
	\begin{align*}
		\abr{J^{\pi_k,d} - J^{\pi_{k-1},d}} &= \abr{\sumsa\mu_{\pi_k}(s)(\pi_k(a|s)-\pi_{k-1}(a|s))q^{\pi_{k-1},d}(s, a)} \leq \sumsa\mu_{\pi_k}(s)\abr{\pi_k(a|s)-\pi_{k-1}(a|s)}|q^{\pi_{k-1},d}(s, a)|\\
		&\leq \sum_{s, a}\mu_{\pi_k}(s)\cdot 8\theta H\iota\pi_{k-1}(a|s)\cdot 6\tmix \leq 48\theta H\iota\tmix.
	\end{align*}
	Next, for any policy $\pi$, define $d^{\pi}(s) = \sum_a\pi(a|s)d(s, a)$.
	Note that by $\inner{\mu_{\pi}}{d^{\pi}}=J^{\pi,d}$,
	\begin{align*}
		v^{\pi,d}(s)=\sum_{t=0}^{\infty}\inner{(P^{\pi})^t_{s,\cdot} - \mu_{\pi}}{d^{\pi}} = \sum_{t=0}^{N-1}\inner{(P^{\pi})^t_{s,\cdot}}{d^{\pi}} - NJ^{\pi, d} + \sum_{t=N}^{\infty}\inner{(P^{\pi})^t_{s,\cdot} - \mu_{\pi}}{d^{\pi}}.
	\end{align*}
	Moreover, by \pref{lem:sum mix}, we have $\abr{\sum_{t=N}^{\infty}\inner{(P^{\pi})^t_{s,\cdot} - \mu_{\pi}}{d^{\pi}} } \leq \frac{1}{T^3}$ for any policy $\pi$.
	Therefore,
	\begin{align*}
		&|v^{\pi_k,d}(s) - v^{\pi_{k-1},d}(s)|\\
		&\leq \abr{\sum_{t=0}^{N-1}\inner{(P^{\pi_k})^t_{s,\cdot} - (P^{\pi_{k-1}})^t_{s,\cdot} }{d^{\pi_k}} } + \abr{\sum_{t=0}^{N-1}\inner{(P^{\pi_{k-1}})^t_{s,\cdot}}{d^{\pi_k}-d^{\pi_{k-1}}}} + N\abr{J^{\pi_k,d} - J^{\pi_{k-1},d}} + \frac{2}{T^3}\\
		&\leq \sum_{t=0}^{N-1}\norm{ (P^{\pi_k})^t - (P^{\pi_{k-1}})^t)d^{\pi_k} }_{\infty} + \sum_{t=0}^{N-1}\norm{d^{\pi_k}-d^{\pi_{k-1}}}_{\infty} + 48\theta HN\iota\tmix + \frac{2}{T^3}.
	\end{align*}
	For the first term, note that:
	\begin{align*}
		\norm{ ((P^{\pi_k})^t - (P^{\pi_{k-1}})^t)d^{\pi_k} }_{\infty} &\leq \norm{ P^{\pi_k}((P^{\pi_k})^{t-1} - (P^{\pi_{k-1}})^{t-1})d^{\pi_k} }_{\infty} + \norm{ (P^{\pi_k} - P^{\pi_{k-1}})(P^{\pi_{k-1}})^{t-1}d^{\pi_k} }_{\infty}\\
		&\leq \norm{ ((P^{\pi_k})^{t-1} - (P^{\pi_{k-1}})^{t-1})d^{\pi_k} }_{\infty} + \max_s\norm{P^{\pi_k}_{s,\cdot}-P^{\pi_{k-1}}_{s,\cdot}}_1. \tag{every row of $P^{\pi_k}$ sums to $1$ and $\norm{(P^{\pi_{k-1}})^{t-1}d^{\pi_k}}_{\infty}\leq 1$}
	\end{align*}
	Moreover, by $\abr{\pi_k(a|s)-\pi_{k-1}(a|s)} \leq 8\theta H\iota\pi_{k-1}(a|s)$,
	\begin{align*}
		\max_s\norm{P^{\pi_k}_{s,\cdot}-P^{\pi_{k-1}}_{s,\cdot}}_1 = \max_s\abr{ \sum_{s'}\sum_a(\pi_k(a|s)-\pi_{k-1}(a|s))P_{s, a}(s') } \leq 8\theta H\iota.
	\end{align*}
	Plugging this back and by a recursive argument, we get $\norm{ ((P^{\pi_k})^t - (P^{\pi_{k-1}})^t)d^{\pi_k} }_{\infty}\leq 8\theta H\iota t$.
	Moreover, $\sum_{t=0}^{N-1}\norm{d^{\pi_k}-d^{\pi_{k-1}}}_{\infty}\leq\sum_{t=0}^{N-1}\max_s\norm{\pi_k(\cdot|s)-\pi_{k-1}(\cdot|s)}_1\leq 8\theta HN\iota$.
	Thus,
	\begin{align*}
		|v^{\pi_k,d}(s) - v^{\pi_{k-1},d}(s)| \leq 8\theta HN^2\iota + 8\theta HN\iota + 48\theta HN\iota\tmix + 2/T^3 \leq 65\theta HN^2\iota. \tag{$\norm{d^{\pi_k}-d^{\pi_{k-1}}}_{\infty}\leq \max_s|\sum_a(\pi_k(a|s)-\pi_{k-1}(a|s))d(s, a)|\leq 8\theta H\iota$}
	\end{align*}
	This completes the proof of the second statement.
	%Next, following the arguments in \citep[Lemma 7]{wei2020model}, the second statement is proved.
\end{proof}

\subsection{\pfref{lem:po}}
\begin{proof}
	Define policy $\pi$ such that $\pi(a|s)=(1-\frac{A}{T})\piref(a|s)+\frac{1}{T}$.
	Clearly, $\pi\in\bar{\Delta}$.
	Moreover, by \pref{lem:val diff} and \pref{lem:bound v}:
	\begin{align*}
		\sum_{k\in\calI} (J^{\piref, r - \frac{\lambda_k}{\eta}c} - J^{\pi, r - \frac{\lambda_k}{\eta}c}) &= \sum_{k\in\calI}\sum_{s, a}\mu_{\pi}(s)(\piref(a|s) - \pi(a|s))q^{\piref, r - \frac{\lambda_k}{\eta}c}(s, a)\\ 
		&\leq \sum_{k\in\calI}\sum_{s, a}\mu_{\pi}(s)\rbr{\frac{A}{T}\piref(a|s) - \frac{1}{T}}q^{\piref, r - \frac{\lambda_k}{\eta}c}(s, a) \leq \frac{12A\lambda}{\eta}.
	\end{align*}
	Therefore,
	\begin{align*}
		&\sum_{k\in\calI} (J^{\piref, r - \frac{\lambda_k}{\eta}c} - J^{\pi_k, r - \frac{\lambda_k}{\eta}c}) = \sum_{k\in\calI} (J^{\piref, r - \frac{\lambda_k}{\eta}c} - J^{\pi, r - \frac{\lambda_k}{\eta}c}) + \sum_{k\in\calI} (J^{\pi, r - \frac{\lambda_k}{\eta}c} - J^{\pi_k, r - \frac{\lambda_k}{\eta}c}) \\
		&\leq \frac{12A\lambda}{\eta} + \sum_{k\in\calI}\sumsa\mu_{\pi}(s)(\pi(a|s) - \pi_k(a|s))(q^{\pi_k, r - \frac{\lambda_k}{\eta}c}(s, a) + (N+1)J^{\pi_k, r - \frac{\lambda_k}{\eta}c}) \tag{\pref{lem:val diff} and $\sum_a(\pi(a|s) - \pi_k(a|s))(N+1)J^{\pi_k, r - \frac{\lambda_k}{\eta}c}=0$}\\
		&= \frac{12A\lambda}{\eta} + \sum_{k\in\calI}\sumsa\mu_{\pi}(s)(\pi(a|s) - \pi_k(a|s))\hatbeta_k(s, a)\\ 
		&\qquad + \sum_{k\in\calI}\sumsa\mu_{\pi}(s)(\pi(a|s) - \pi_k(a|s))\rbr{P_{s, a}(v^{\pi_k, r - \frac{\lambda_k}{\eta}c} + NJ^{\pi_k, r - \frac{\lambda_k}{\eta}c}) - \P_{k, s, a}\hatV_k }\\
		&= \frac{12A\lambda}{\eta} + \sum_{k\in\calI}\sumsa\mu_{\pi}(s)(\pi(a|s) - \pi_k(a|s))\hatbeta_k(s, a)\\
		&\qquad + \sum_{k\in\calI}\sumsa\mu_{\pi}(s)(\pi(a|s) - \pi_k(a|s))P_{s, a}\rbr{v^{\pi_k, r - \frac{\lambda_k}{\eta}c} + NJ^{\pi_k, r - \frac{\lambda_k}{\eta}c} - \hatV_k}\\
		&\qquad + \sum_{k\in\calI}\sumsa\mu_{\pi}(s)(\pi(a|s) - \pi_k(a|s))(P_{s, a} - \P_{k, s, a})\hatV_k.
	\end{align*}
	For the third term above, by \pref{lem:hatVk}, \pref{lem:azuma}, and $\norm{\hatV_k}_{\infty}\leq \frac{2\lambda N}{\eta}$, with probability at least $1-\delta$,
	\begin{align*}
		&\sum_{k\in\calI}\sumsa\mu_{\pi}(s)(\pi(a|s) - \pi_k(a|s))P_{s, a}\rbr{v^{\pi_k, r - \frac{\lambda_k}{\eta}c} + NJ^{\pi_k, r - \frac{\lambda_k}{\eta}c} - \hatV_k}\\
		&\leq \sum_{k\in\calI}\sumsa\mu_{\pi}(s)(\pi(a|s) - \pi_k(a|s))P_{s, a}\rbr{\E_k[\hatV_k] - \hatV_k} + \frac{\lambda}{\eta} \leq \frac{4\lambda N}{\eta}\sqrt{|\calI|\ln\frac{4T^3}{\delta}}.
	\end{align*}
	For the fourth term above, with probability at least $1-\delta$:
	\begin{align*}
		&\sum_{k\in\calI}\sumsa\mu_{\pi}(s)(\pi(a|s) - \pi_k(a|s))(P_{s, a} - \P_{k, s, a})\hatV_k \leq \sum_{k\in\calI}\sumsa\mu_{\pi}(s)(\pi(a|s) + \pi_k(a|s))\norm{P_{s, a} - \P_{k, s, a}}_1\norm{\hatV_k}_{\infty} \tag{Cauchy-Schwarz inequality}\\
		&\leq \sum_{k\in\calI}\sumsa\mu_{\pi}(s, a)\frac{2\lambda N}{\eta}\sqrt{S\ln\frac{2SAT}{\delta}}\rbr{\frac{1}{\sqrt{\Np_k(s, a)}} + \sum_{a'}\frac{\pi_k(a'|s)}{\sqrt{\Np_k(s, a')}}} \tag{\pref{lem:weiss} with a union bound over $\SA\times[T]$ and $\norm{\hatV_k}_{\infty}\leq \frac{2\lambda N}{\eta}$}\\
		&= \sum_{k\in\calI}\sumsa\mu_{\pi}(s, a)x_k(s, a) = \sum_{k\in\calI}J^{\pi, P, x_k} = \sum_{k\in\calI}(J^{\pi, P, x_k} - J^{\pi_k, P_k, x_k}) + \sum_{k\in\calI}J^{\pi_k, P_k, x_k}.
	\end{align*}
	%Note that $J^{\pi_k,P_k,x_k}$ is state independent since $P_k$ can be treated as an optimal policy of some augmented MDP, and $(P_{k,s,a}-P_{s, a})v^{\pi_k,P_k,x_k}\geq 0$ for all $(s, a)\in\SA$ by its optimality (see \pref{app:Pk}).
	%Thus by \pref{lem:val diff},
	By \pref{lem:approx val diff},
	\begin{align*}
		J^{\pi, P, x_k} - J^{\pi_k, P_k, x_k} &\leq \sum_{s, a}\mu_{\pi}(s)(\pi(a|s) - \pi_k(a|s))u_k(s, a)  + \sum_{s, a}\mu_{\pi}(s, a)(P_{s, a} - P_{k, s, a})u'_k + \epsevi\\
		&\leq \sum_{s, a}\mu_{\pi}(s)(\pi(a|s) - \pi_k(a|s))u_k(s, a) + \epsevi. \tag{definition of $P_{k,s,a}$}
	\end{align*}
	Substituting these back, we have:
	\begin{align*}
		&\sum_{k\in\calI} J^{\piref, r - \frac{\lambda_k}{\eta}c} - J^{\pi_k, r - \frac{\lambda_k}{\eta}c}\\
		&\leq \frac{12A\lambda}{\eta} + \sum_{k\in\calI}\sumsa\mu_{\pi}(s)(\pi(a|s) - \pi_k(a|s))(\hatbeta_k(s, a) + u_k(s, a)) + \frac{4\lambda N}{\eta}\sqrt{|\calI|\ln\frac{4T^3}{\delta}} + \sum_{k\in\calI}J^{\pi_k, P_k, x_k} + K\epsilon_{\EVI}.
	\end{align*}
	Note that by the standard OMD analysis~\citep{hazan2019introduction}, for any $s\in\calS$:
	\begin{align*}
		&\sum_{k\in\calI}\suma(\pi(a|s) - \pi_k(a|s))(\hatbeta_k(s, a) + u_k(s, a))\\
		&\leq \sum_{k\in\calI}\rbr{D(\pi(\cdot|s), \pi_k(\cdot|s)) - D(\pi(\cdot|s), \pi_{k+1}(\cdot|s)) + D(\pi_k(\cdot|s), \pi'_{k+1}(\cdot|s))  },
	\end{align*}
	where $\pi'_{k+1}(a|s)=\pi_k(a|s)\exp(\theta (\hatbeta_k(s, a) + u_k(s, a)))$.
	Then by $\theta \abr{\hatbeta_k(s, a) + u_k(s, a)}\leq 2\theta H\iota\leq 1$ (\pref{lem:Q span}):
	\begin{align*}
		D(\pi_k(\cdot|s), \pi'_{k+1}(\cdot|s)) &= \frac{1}{\theta }\suma\rbr{\pi_k(a|s)\ln\frac{\pi_k(a|s)}{\pi'_{k+1}(a|s)} - \pi_k(a|s) + \pi'_{k+1}(a|s) }\\
		&= \frac{1}{\theta }\suma\pi_k(a|s)\rbr{ -\theta (\hatbeta_k(s, a) + u_k(s, a)) -1 + e^{\theta (\hatbeta_k(s, a) + u_k(s, a))} }\\
		&\leq \theta \suma\pi_k(a|s)\rbr{ \hatbeta_k(s, a) + u_k(s, a) }^2. \tag{$e^{-x}-1+x\leq x^2$ for $x\geq -1$}
	\end{align*}
	Therefore,
	\begin{align*}
		&\sum_{k\in\calI}\sumsa\mu_{\pi}(s)(\pi(a|s) - \pi_k(a|s))(\hatbeta_k(s, a) + u_k(s, a))\\
		&\leq \sums\mu_{\pi}(s)\sum_{k\in\calI}\rbr{D(\pi(\cdot|s), \pi_k(\cdot|s)) - D(\pi(\cdot|s), \pi_{k+1}(\cdot|s)) + \theta \suma\pi_k(a|s)\rbr{\hatbeta_k(s, a) + u_k(s, a)}^2  }\\
		&\leq \sums\mu_{\pi}(s)\rbr{ \frac{\ln T}{\theta } + \theta \sum_{k\in\calI}\suma\pi_k(a|s)\rbr{\hatbeta_k(s, a) + u_k(s, a)}^2} \tag{$\pi_k(\cdot|s)\in\bar{\Delta}$}\\
		&\leq \sums\mu_{\pi}(s)\rbr{ \frac{\ln T}{\theta } + 4\theta |\calI|H^2\iota^2 } \leq 4H\iota\sqrt{K\ln T} + 4H\iota\ln T. \tag{\pref{lem:Q span} and definition of $\theta $}
		%= \tilO{\frac{\lambda NH}{\eta}\sqrt{SK}}.
		%\leq \sums\mu_{\pi}(s)\rbr{ \frac{\ln T}{\theta } + \theta |\calI|N^2H^2S\rbr{\frac{\lambda}{\eta}}^2 } = \tilO{\frac{\lambda NH}{\eta}\sqrt{SK}},
	\end{align*}
	Moreover, with probability at least $1-2\delta$,
	\begin{align*}
		&\sum_{k\in\calI}J^{\pi_k, P_k, x_k} = \sum_{k\in\calI}(J^{\pi_k, P_k, x_k} - J^{\pi_k, P, x_k}) + \sum_{k\in\calI}J^{\pi_k, P, x_k}\\
		&\leq \sum_{k\in\calI}\sum_{s,a}\mu_{\pi_k}(s, a)[P_{k, s, a}-P_{s, a}]u'_k + \sum_{k\in\calI}\sum_{s, a}\mu_{\pi_k}(s, a)x_k(s, a) + K\epsevi \tag{\pref{lem:approx val diff}}\\
		&\leq \sum_{k\in\calI}\sumsa\mu_{\pi_k}(s, a)\frac{3H\iota\sqrt{S\ln\frac{2SAT}{\delta}}}{\sqrt{\Np_k(s, a)}} + K\epsevi \tag{\pref{lem:Qk}, definition of $\calQ_k$, and \pref{lem:EVI span}}\\
		&\leq 24\iota\sqrt{S^2AH|\calI|\ln\frac{2SAT}{\delta}} + 60S^{1.5}AH\iota\ln^{3/2}\rbr{\frac{4SAT}{\delta}} + K\epsilon_{\EVI}. \tag{\pref{lem:sum mu}}
		%&= \tilO{\sum_{k\in\calI}\sum_{s, a}\mu_{\pi_k, P}(s, a)\frac{\sqrt{S}}{\sqrt{\Np_k(s, a)}}\frac{\lambda}{\eta}NH\sqrt{S}} = \tilO{ \frac{\lambda N}{\eta}\sqrt{S^3AH|\calI|} + \frac{\lambda NHS^2A}{\eta} }.
		%&\leq \tilO{(\thit + \tmix)\rbr{1 + \frac{\lambda}{\eta}}NS\sqrt{A|\calI|} + \rbr{1 + \frac{\lambda}{\eta}}N\sqrt{SA|\calI|}}.
	\end{align*}
	Substituting these back and by the definition of $\lambda$ completes the proof.
\end{proof}

\subsection{\pfref{thm:ergodic}}
\begin{proof}
	For constraint violation, note that:
	\begin{align*}
		&\sumt (c(s_t, a_t) - \tau) = \sumk\sum_{h=1}^H (c(s^k_h, a^k_h) - J^{\pi_k,c}) + H\sumk(J^{\pi_k,c}-\tau).
	\end{align*}
	For the first term, by \pref{lem:stab episode}, we have $\sumk\sum_{h=1}^H (c(s^k_h, a^k_h) - J^{\pi_k,c})\leq\lambda+\tilO{\tmix}$ with probability at least $1-2\delta$.
	For the second term, by \pref{lem:bound lambda}, we have $\lambda_{k+1}\geq\lambda_k + \hatJ_k + \epsilon - \tau$ with probability at least $1-6\delta$, and with probability at least $1-\delta$,
	\begin{align*}
		H\sumk(J^{\pi_k,c}-\tau) &\leq H\sumk(J^{\pi_k,c} - \E_k[\hatJ_k]) + H\sumk(\E_k[\hatJ_k] - \hatJ_k) + H\sumk(\lambda_{k+1}-\lambda_k-\epsilon)\\
		&\leq \frac{1}{T} + H\sqrt{2K\ln\frac{4T^3}{\delta}} + H(\lambda - K\epsilon) \tag{\pref{lem:hatJ} and \pref{lem:azuma}}\\
		&\leq \tilO{1} + \lambda + H(\lambda - K\epsilon). \tag{definition of $\lambda$}
	\end{align*}
	When $\epsilon=3\lambda/K$, that is, $3\lambda/K\geq(\tau-c^0)/2$, we have
	%when $K\geq\frac{N^2H^2S^3A}{(\tau-c^0)^4}$, we have $\epsilon\leq (\tau-c^0)/2$ and:
	\begin{align*}
		\sumt (c(s_t, a_t) - \tau) \leq 2\lambda + \tilO{\tmix} + H(\lambda - K\epsilon) = \tilO{\tmix}.
	\end{align*}
	Otherwise, $K\leq \frac{6\lambda}{\tau-c^0}$, which gives $T=\tilO{\frac{N^2H^2S^3A + N^2H^3S}{(\tau-c^0)^2} + \frac{NH^2S^2A}{\tau-c^0}}$ and the constraint violation is of the same order by $C_T\leq T$.
	%Plugging in the definition of $N$ and $H$ gives the desired bound.
	 
	For regret, note that $\sumk (\optJeps - J^{\pi_k,r}) = \sumk (J^{\optpieps, r - \frac{\lambda_k}{\eta}c} - J^{\pi_k, r - \frac{\lambda_k}{\eta}c}) + \sumk\frac{\lambda_k}{\eta}(J^{\optpieps, c} - J^{\pi_k, c})$, and with probability at least $1-\delta$,
	\begin{align*}
		\sumk\frac{\lambda_k}{\eta}(J^{\optpieps, c} - J^{\pi_k, c}) &\leq \sumk\frac{\lambda_k}{\eta}\rbr{\tau - \epsilon - J^{\pi_k, c}} \tag{definition of $\optpieps$}\\
		&\leq \sumk\frac{\lambda_k}{\eta}\rbr{\tau - \epsilon - \hatJ_k} + \sumk\frac{\lambda_k}{\eta}(\hatJ_k - \E_k[\hatJ_k]) + \frac{\lambda}{\eta T}\tag{\pref{lem:hatJ}}\\
		&\leq \frac{1}{\eta}\sumk\lambda_k(\lambda_k-\lambda_{k+1}) + \frac{\tau^2K}{\eta} + \tilO{\frac{\lambda}{\eta}\sqrt{K}},
		%\tag{$\lambda_{k+1}>0$ if $\lambda_k>\tau$ and \pref{lem:azuma}}
	\end{align*}
	where in the last step we apply Azuma's inequality and the following argument: if $\lambda_{k+1}>0$, then $\tau-\epsilon-\hatJ_k = \lambda_k-\lambda_{k+1}$ by the definition of $\lambda_k$.
	Otherwise, $\lambda_k\leq\tau-\epsilon-\hatJ_k<\tau$ and $\lambda_k(\tau-\epsilon-\hatJ_k)\leq\tau^2$.
	Moreover, $\sumk\lambda_k(\lambda_k-\lambda_{k+1}) = \frac{1}{2}\sumk\rbr{ \lambda_k^2 - \lambda_{k+1}^2 + (\lambda_{k+1} - \lambda_k)^2 } \leq \frac{K}{2}$ by $\lambda_1=0$ and $|\lambda_k-\lambda_{k+1}|\leq 1$.
	Therefore, by \pref{lem:po} and definition of $\lambda$ and $\eta$, with probability at least $1-4\delta$,
	\begin{align*}
		\sumk (\optJeps - J^{\pi_k,r}) &\leq \sumk (J^{\optpieps, r - \frac{\lambda_k}{\eta}c} - J^{\pi_k, r - \frac{\lambda_k}{\eta}c}) + \tilO{\frac{K}{\eta} + \frac{\lambda}{\eta}\sqrt{K}}\\
		&= \tilO{\frac{N}{\tau-c^0}\rbr{\sqrt{S^3AT} + \sqrt{SHT} + S^2AH} + \frac{1}{(\tau-c^0)^2}}.
		%= \tilO{  \frac{N}{\tau-c^0}\sqrt{S^3AT} + \frac{N}{\tau-c^0}\sqrt{SHT} + \frac{S^2ANH}{\tau-c^0} }.
		%&= \tilO{ \frac{NH}{\tau-c^0}\sqrt{S^3AK} }.
	\end{align*}
	Thus, with probability at least $1-2\delta$,
	\begin{align*}
		&\sumt (\optJ - r(s_t, a_t)) = H\sumk (\optJ - \optJeps) + H\sumk (\optJeps - J^{\pi_k,r}) + \sumk\sum_{h=1}^H (J^{\pi_k,r} - r(s^k_h, a^k_h))\\
		&\leq \frac{T\epsilon}{\tau-c^0} + \tilO{ \frac{NH}{\tau-c^0}\rbr{\sqrt{S^3AT} + \sqrt{SHT} + S^2AH} + \frac{H}{(\tau-c^0)^2} } \tag{\pref{lem:Jeps} and \pref{lem:stab episode}}\\
		&= \tilO{ \frac{NH}{\tau-c^0}\rbr{\sqrt{S^3AT} + \sqrt{SHT} + S^2AH} + \frac{H}{(\tau-c^0)^2} }.\tag{by the definition of $\epsilon$}
		%&= \tilO{ \frac{NH}{(\tau-c^0)^2}\sqrt{S^3AHT} + N^2\sqrt{K}} = \tilO{ \frac{NH}{(\tau-c^0)^2}\sqrt{S^3AHT}}.
	\end{align*}
	Plugging in the definition of $N$ and $H$ completes the proof.
\end{proof}

\subsection{Transition Estimation and Computation of $u_k, P_k$}
\label{app:Pk}
Define $\calQ_k$ as a transition confidence set based on Weissman's inequality (\pref{lem:weiss}), such that $\calQ_k=\cap_{s,a}\calQ_{k,s,a}$, and
\begin{align*}
	\calQ_{k,s,a} = \cbr{P': \norm{P'_{s, a} - \P_{k, s, a}}_1 \leq \sqrt{\frac{S\ln\frac{2SAT}{\delta}}{\Np_k(s, a)}}}.
\end{align*}
We first show that $P$ falls in $\calQ_k$ with high probability.
\begin{lemma}
	\label{lem:Qk}
	With probability at least $1-\delta$, $P\in\calQ_k$ for all $k$.
\end{lemma}
\begin{proof}
	For any $(s, a)$, $n\in[T]$ and $m=S$, \pref{lem:weiss} gives with probability at least $1-\frac{\delta}{SAT}$: $\norm{P_{s, a}-\P^n_{s, a}}\leq \sqrt{\frac{S\ln\frac{2SAT}{\delta}}{n}}$, where $\P^n_{s, a}$ is the empirical distribution computed by $n$ i.i.d samples from $P_{s, a}$.
	Taking a union bound over $(s, a)\in\SA$, $n\in[T]$ proves the statement.
\end{proof}
Next, we show the computation procedure of $u_k$ and $P_k$ for a fixed episode $k$.
Note that $P_k$ is an approximation of $P^{\star}_k=\argmax_{P'\in\calQ_k}J^{\pi_k,P',x_k}$, and finding $P^{\star}_k$ is equivalent to computing the optimal policy in an extended MDP $\tilcalM_k$ with state space $\calS$ and extended action space $\calQ_k$, such that for any extended action $P'\in\calQ_k$, the reward at $(s, P')$ is $\suma\pi_k(a|s)r(s, a)$ and the transition probability to $s'$ is $\sum_a\pi_k(a|s)P'_{s, a}(s')$. 
Note that since $\calQ_k=\bigcap_{s, a}\calQ_{k,s,a}$ where $\calQ_{k,s,a}$ only puts constraints on transition at $(s, a)$, any deterministic policy in $\tilcalM_k$ can also be represented by an element in $\calQ_k$.
We adopt a variant of Extended Value Iteration (EVI) in \citep[Theorem 7]{jaksch2010near} to approximate $P^{\star}_k$, where we execute the following value iteration procedure in $\tilcalM_k$,
\begin{equation}
	u^0(s) = 0, \quad u^{i+1}(s) = \suma\pi_k(a|s)\rbr{x_k(s, a) + \max_{P\in\calQ_k}P_{s, a}u^i}. \label{eq:EVI}
\end{equation}
We stop the iteration above at index $\istar$, which is the first index $i$ such that $\sp(u^{i+1}-u^i)\leq\epsevi = \frac{1}{T}$.
Then we define $u_k(s)=u^{\istar+1}(s)-\min_{s'}u^{\istar+1}(s')$, $u'_k(s)=u^{\istar}(s)-\min_{s'}u^{\istar}(s')$, $P_{k,s,a}=\argmax_{P\in\calQ_k}P_{s,a}u^{\istar}$ as the transition in $u_k$, and $u_k(s, a)=x_k(s,a) - \min_{s'}u^{\istar+1}(s') + P_{k,s,a}u^{\istar}$ so that $u_k(s)=\sum_a\pi_k(a|s)u_k(s,a)$, which is the function used in \pref{alg:ergodic}.
%Then we define $u_k(s)=u^{\istar+1}(s)-\min_{s'}u^{\istar+1}(s')$, $u'_k(s)=u^{\istar}(s)-\min_{s'}u^{\istar}(s')$, and $P_{k,s,a}=\argmax_{P\in\calQ_k}P_{s,a}u'_k$ as the transition in $u_k$.
%Also define $u_k(s, a)=x_k(s, a) + P_{k,s,a}u'_k$ so that $u_k(s)=\sum_a\pi_k(a|s)u_k(s,a)$, which is the function used in \pref{alg:ergodic}.
Also note that the maximization in \pref{eq:EVI} can be solved by \citep[Figure 2]{jaksch2010near}.

Now we show that the value iteration in \pref{eq:EVI} always converges (specifically, the transition converges to $P^{\star}_k$) similar to \citep[Theorem 7]{jaksch2010near}.
First note that $\tilcalM_k$ is communicating since $P\in\calQ_k$ whose corresponding MDP is ergodic.
Moreover, the transition chosen in each iteration of \pref{eq:EVI} is aperiodic and unichain.
This is because in each iteration of \pref{eq:EVI} there is a ``best'' state and every state has non-zero probability transiting to the ``best'' state.
Following the proof of \citep[Theorem 7]{jaksch2010near}, we conclude that EVI in \pref{eq:EVI} converges.
%On convergence of \pref{eq:EVI}, we also have $u_i(s)-\min_su_i(s)$ converges to $v^{\pi_k,P_k,x_k}(s)-\min_sv^{\pi_k,P_k,x_k}(s)$ by \citep[Remark 8]{jaksch2010near}, which means $\lim_{i\rightarrow\infty}\sp(u_i)=\sp(v^{\pi_k,P_k,x_k})$.
%Thus, \pref{eq:EVI} implies $(P_{k,s,a}-P'_{s,a})v^{\pi_k,P_k,x_k}$ for any $P'\in \calQ_k$ and $(s, a)\in\SA$.
%Thus, we have $P_{k,s,a}v^{\pi_k,P_k,x_k}\geq P_{s, a}v^{\pi_k,P_k,x_k}$.
%Now taking a small enough $\epsilon$ ($1/T$ for example) and substitute $q^{\pi_k,P_k,x_k}$ by $q^{\pi_k,P'_k,x_k}$ in \pref{alg:ergodic}, we still obtain the regret guarantee in \pref{thm:ergodic}.
%Moreover, the transition $P_k^{\pi_k}$ is aperiodic and unichain since in each iteration of \pref{eq:EVI} there is a ``best'' state and every state has non-zero probability transiting to the ``best'' state.

Since $P_k$ is aperiodic and unichain by the arguments above, there exist constant $J^{\pi_k,P_k,x_k}$ such that $J^{\pi_k,P_k,x_k}(s)=J^{\pi_k,P_k,x_k}$. 
Importantly, by \citep[Theorem 8.5.6]{puterman1994markov}, we have:
\begin{equation}
	\label{eq:uJ eps}
	|u^{\istar+1}(s)-u^{\istar}(s)-J^{\pi_k,P_k,x_k}|\leq\epsevi
\end{equation}
This leads to the following approximated value difference lemma.
\begin{lemma}
	\label{lem:approx val diff}
	$J^{\pi,P,x_k}-J^{\pi_k,P_k,x_k}=\sumsa\mu_{\pi}(s)\rbr{\pi(a|s)-\pi_k(a|s)}u_k(s, a) + \sumsa\mu_{\pi}(s, a)(P_{s,a}-P_{k,s,a})u'_k + \delta_{\EVI}$, where $|\delta_{\EVI}|\leq \epsevi$.
\end{lemma}
\begin{proof}
	Define $u^i(s, a)=x_k(s, a) + \max_{P\in\calQ_k}P_{s, a}u^{i-1}$ so that $u^i(s)=\sum_a\pi_k(a|s)u^i(s, a)$.
	Since $P$ is ergodic, we have:
	\begin{align*}
		J^{\pi,P,x_k} &= \sumsa\mu_{\pi}(s, a)x_k(s, a) = \sumsa\mu_{\pi}(s, a)\rbr{ u^{\istar+1}(s, a) - P_{k,s,a}u^{\istar} }\\
		&= \sumsa\mu_{\pi}(s, a)\rbr{ u^{\istar+1}(s, a) - u^{\istar+1}(s) } + \sumsa\mu_{\pi}(s, a)\rbr{u^{\istar+1}(s) - P_{s, a}u^{\istar}} + \sumsa\mu_{\pi}(s, a)(P_{s,a}-P_{k,s,a})u^{\istar}\\
		&= \sumsa\mu_{\pi}(s, a)\rbr{ u_k(s, a) - u_k(s) } + \sumsa\mu_{\pi}(s, a)\rbr{u^{\istar+1}(s) - P_{s, a}u^{\istar}} + \sumsa\mu_{\pi}(s, a)(P_{s,a}-P_{k,s,a})u'_k.
	\end{align*}
	Let $\delta_{\EVI}=\sumsa\mu_{\pi}(s, a)\rbr{u^{\istar+1}(s) - P_{s, a}u^{\istar}}$. 
	By \pref{eq:uJ eps}, we have
	\begin{align*}
		\delta_{\EVI}&\leq\sumsa\mu_{\pi}(s, a)\rbr{ u^{\istar}(s) - P_{s,a}u^{\istar} + J^{\pi_k,P_k,x_k} + \epsevi } = J^{\pi_k,P_k,x_k} + \epsevi. \tag{$\mu_{\pi}(s') = \sumsa\mu_{\pi}(s,a)P_{s,a}(s')$}
	\end{align*}
	Showing $\delta_{\EVI}\geq-\epsevi$ is similar.
	Further by $\sum_a\mu_{\pi}(s,a)(u_k(s,a)-u_k(s))=\mu_{\pi}(s)\sum_a(\pi(a|s)-\pi_k(a|s))u_k(s,a)$, the statement is proved.
\end{proof}

\subsection{Auxiliary Lemmas}
\begin{lemma}\citep[Corollary 13.1]{wei2020model}
	\label{lem:mix}
	For any ergodic MDP with mixing time $\tmix$, we have $\norm{(P^{\pi})^t_{s,\cdot} - \mu_{\pi}}_1\leq 2\cdot 2^{-t/\tmix}$ for all policy $\pi$, state $s\in\calS$, and $t\geq2\tmix$.
\end{lemma}

\begin{lemma}\citep[Corollary 13.2]{wei2020model}
	\label{lem:sum mix}
	Let $N=4\tmix\log_2T$.
	For an ergodic MDP with mixing time $\tmix<T/4$, we have for all $\pi$: $\sum_{t=N}^{\infty}\norm{(P^{\pi})^t_{s,\cdot} - \mu_{\pi}}_1\leq \frac{1}{T^3}$. 
\end{lemma}

\begin{lemma}\citep[Lemma 14]{wei2020model}
	\label{lem:bound v}
	For an ergodic MDP with mixing time $\tmix$, utility function $d\in[0, 1]^{\SA}$, and any $\pi, s, a$, $|v^{\pi,d}(s)|\leq 5\tmix$ and $|q^{\pi,d}(s, a)|\leq 6\tmix$.
\end{lemma}

\begin{lemma}
	\label{lem:EVI span}
	Under the event of \pref{lem:Qk},
	$\max\{\sp(u_k), \sp(u'_k)\}\leq 4\tmix\thit\ceil{\log_2(4\thit)}\iota\leq \frac{H\iota}{4}$.
\end{lemma}
\begin{proof}
	%Denote by $V^P_H$ the value function w.r.t policy $\pi_k$, transition $P$ and reward $x_k$ in a finite-horizon MDP with horizon $H$, and let $\optV_H(s)=\argmax_{P'\in\calP_k}V^{P'}_H(s)$ for any $s$, which is the intermediate value function after $H$ steps of EVI (see \pref{app:Pk}) such that $\lim_{H\rightarrow\infty}\optV_H(s) - \optV_H(s') = v^{\pi_k,P_k,x_k}(s)-v^{\pi_k,P_k,x_k}(s')$.
	For a fixed $k$, it suffices to show that for any two states $s,s'$ and $H\geq1$, $u^H(s) - u^H(s') = \tilo{(\lambda/\eta)\thit\tmix}$, where $u^i$ defined in \pref{eq:EVI} is the optimal value function of taking $H$ steps in $\tilcalM_k$.
	Without loss of generality, assume $u^H(s) \geq u^H(s')$.
	Define random variable $\tau$ as the number of steps it takes to transits from state $s'$ to $s$.
	Then by \pref{lem:Qk}, $u^H(s') \geq \E_{\tau}[u^{H-\min\{H,\tau\}}(s)|\pi_k, P]$ (the right hand side is a lower bound of the expected reward of a history-dependent policy in $\tilcalM_k$ which follows $P$ at first and then switches to $P_k$ when reaching $s$, and it is dominated by $u_H(s')$ by the Markov property).
	Thus by $x_k(s,a)\leq2\iota$,
	\begin{align*}
		u^H(s) - u^H(s') &\leq u^H(s) - \E_{\tau}[u^{H-\min\{H,\tau\}}(s)|\pi_k, P] = \E_{\tau}[u^H(s) - u^{H-\min\{H,\tau\}}(s)|\pi_k, P]\\ 
		&\leq 2\E_{\tau}[\tau|\pi_k, P]\iota \leq 4\tmix\thit\ceil{\log_2(4\thit)}\iota.
	\end{align*}
	For the last inequality above, note that when $t=\tmix\ceil{\log_2(4\thit)}$, we have $\norm{(P^{\pi_k})^t_{s,\cdot} - \mu_{\pi_k}}_{\infty} \leq \frac{1}{2\thit}$ for any $s\in\calS$ by \pref{lem:mix}.
	Therefore, $(P^{\pi_k})^t_{s, s'} \geq \frac{1}{2}\mu_{\pi_k}(s')\geq\frac{1}{2\thit}$ for any $s, s'\in\calS$.
	This implies that we can reach any state at least once by taking $2t\cdot\thit$ steps in expectation, that is, $\E_{\tau}[\tau|\pi_k,P]\leq 2t\cdot\thit$.
	The second inequality in the statement follows directly from the definition of $H$.
	%Then by $\lim_{i\rightarrow\infty}\sp(u_i)=\sp(v^{\pi_k,P_k,x_k})$ (see \pref{app:Pk}), the first inequality is proved.
\end{proof}

\begin{lemma}
	\label{lem:Q span}
	Under the event of \pref{lem:Qk},
	$|\hatbeta_k(s, a) + u_k(s, a)|\leq2H\iota$ for all $(s, a)\in\SA$.
\end{lemma}
\begin{proof}
	Note that $\hatbeta_k(s, a)\leq \frac{2\lambda}{\eta}(N+1) \leq H\iota$.
	Define $\sstar=\argmin_su^{\istar+1}(s)$. 
	By \pref{lem:EVI span},
	\begin{align*}
		u_k(s, a) &= x_k(s, a) + P_{k,s,a}u^{\istar} - u^{\istar+1}(\sstar) = x_k(s, a) - \sum_{a'}\pi_k(a'|\sstar)x_k(\sstar, a') + (P_{k,s,a} - (P^{\pi_k}_k)_{\sstar,\cdot})u^{\istar}\\
		&\leq 2\iota + \sp(u^{\istar}) \leq H\iota.
	\end{align*} 
	This completes the proof.
	%and $u_k(s, a)\leq H\iota$ by .
\end{proof}

\begin{lemma}\citep[Lemma 15]{wei2020model}
\label{lem:val diff}
	For any two policies $\pi,\pi'$ and utility function $d$, 
	$$J^{\pi,d}-J^{\pi',d}=\sumsa\mu_{\pi}(s)(\pi(a|s)-\pi'(a|s))q^{\pi',d}(s, a).$$
\end{lemma}

\begin{lemma}
	\label{lem:sum mu}
	With probability at least $1-\delta$, for any $\calI\subseteq[K]$, $\sum_{k\in\calI}\sumsa\frac{\mu_{\pi_k}(s, a)}{\sqrt{\Np_k(s, a)}}=8\sqrt{SA|\calI|/H} + 20SA\ln\frac{4T}{\delta}$.
\end{lemma}
\begin{proof}
	Define $n_k(s, a)=\sum_{h=N+1}^H\Ind\{s^k_h=s, a^k_h=a\}$.
	Note that $\E_k[n_k(s, a)] = \sum_{h=N}^{H-1}(P^{\pi_k})^h_{s^k_1, s}\pi_k(a|s)$.
	Moreover, by \pref{lem:mix}, for any state $s'$ and $h\geq N$: $\mu_{\pi_k}(s, a) - (P^{\pi_k})^h_{s', s}\pi_k(a|s) \leq 1/T^2$.
	Therefore,
	\begin{align*}
		\sum_{k\in\calI}\sumsa\frac{\mu_{\pi_k}(s, a)}{\sqrt{\Np_k(s, a)}} &\leq \sum_{k\in\calI}\sumsa\frac{1}{T^2} + \frac{1}{H-N}\sum_{k\in\calI}\sumsa\sum_{h=N}^{H-1}\frac{(P^{\pi_k})^h_{s^k_1, s}\pi_k(a|s)}{\sqrt{\Np_k(s, a)}}\\
		&\leq SA/(HT) + \frac{2}{H-N}\sum_{k\in\calI}\sumsa\frac{n_k(s, a)}{\sqrt{\Np_k(s, a)}} + 16\ln\frac{4T}{\delta} \tag{\pref{lem:e2r} with a union bound over $T$ possible values of $\calI_{[1]}$}\\
		&\leq \frac{2}{H-N}\sum_{k\in\calI}\sumsa\frac{n_k(s, a)}{\sqrt{\Np_{k+1}(s, a)}} +  2\sum_{k\in\calI}\sumsa\rbr{\frac{1}{\sqrt{\Np_k(s, a)}}-\frac{1}{\sqrt{\Np_{k+1}(s, a)}}} + 17SA\ln\frac{4T}{\delta}\\
		&\leq \frac{8}{H}\sqrt{SAH|\calI|} + 20SA\ln\frac{4T}{\delta} = 8\sqrt{SA|\calI|/H} + 20SA\ln\frac{4T}{\delta},
	\end{align*}
	where the last inequality is by $\sum_{k\in\calI}\frac{n_k(s,a)}{\sqrt{\Np_{k+1}(s, a)}}\leq 2\sqrt{\sum_{k\in\calI}n_k(s,a)}$, Cauchy Schwarz inequality, and $H-N\geq\frac{H}{2}$.
\end{proof}

\begin{lemma}
	\label{lem:stab episode}
	For any utility function $d\in[0, 1]^{\SA}$, we have $\abr{\sumk\sum_{h=1}^H d(s^k_h, a^k_h) - J^{\pi_k, d}}\leq 12\tmix\sqrt{2T\ln\frac{4T^3}{\delta}} + 33N^2\sqrt{K\ln T} + \tilO{\tmix}\leq \lambda + \tilO{\tmix}$ with probability at least $1-2\delta$.
\end{lemma}
\begin{proof}
	With probability at least $1-2\delta$,
	\begin{align*}
		&\abr{\sumk\sum_{h=1}^H (d(s^k_h, a^k_h) - J^{\pi_k,d})}= \abr{\sumk\sum_{h=1}^H\rbr{q^{\pi_k,d}(s^k_h, a^k_h) - P^k_hv^{\pi_k,d}}} \tag{\pref{eq:Bellman}}\\
		&=\abr{\sumk\sum_{h=1}^H\rbr{q^{\pi_k,d}(s^k_h, a^k_h) - v^{\pi_k,d}(s^k_h)} + \sumk(v^{\pi_k,d}(s^k_1) - v^{\pi_k,d}(s^k_{H+1})) + \sumk\sum_{h=1}^H(v^{\pi_k,d}(s^k_{h+1}) - P^k_hv^{\pi_k,d})} \tag{$\sumh(v^{\pi_k,d}(s^k_h)-v^{\pi_k,d}(s^k_{h+1}))=v^{\pi_k,d}(s^k_1)-v^{\pi_k,d}(s^k_{H+1})$}\\
		&\leq \abr{\sumk\sum_{h=1}^H\rbr{q^{\pi_k,d}(s^k_h, a^k_h) - v^{\pi_k,d}(s^k_h)} + \sum_{k=2}^K(v^{\pi_k,d}(s^k_1) - v^{\pi_{k-1},d}(s^k_1)) + \sumk\sum_{h=1}^H(v^{\pi_k,d}(s^k_{h+1}) - P^k_hv^{\pi_k,d})} + \tilO{\tmix} \tag{$s^k_{H+1}=s^{k+1}_1$ and \pref{lem:bound v}}\\
		&\leq 12\tmix\sqrt{2T\ln\frac{4T^3}{\delta}} + 65\theta HN^2K\iota + \tilO{\tmix} \leq 12\tmix\sqrt{2T\ln\frac{4T^3}{\delta}} + 33N^2\sqrt{K\ln T} + \tilO{\tmix}. \tag{\pref{lem:stab}, \pref{lem:bound v}, and \pref{lem:azuma}}
	\end{align*}
	The second inequality directly follows from the definition of $\lambda$.
\end{proof}

\begin{lemma}
	\label{lem:bound lambda}
	With probability at least $1-6\delta$, $\lambda_k<\lambda$ for any $k$, that is, the upper bound truncation of $\lambda_k$ is never triggered.
\end{lemma}
\begin{proof}
	We prove this by induction on $k$.
	The base case $k=1$ is clearly true.
	For $k>1$, if $\lambda_k\leq \frac{2(\eta+1)}{\tau-c^0}$, the statement is proved.
	Otherwise, let $j=\max\{j' < k: \lambda_{j'}\leq\frac{2(\eta+1)}{\tau-c^0}, \lambda_{j'+1}>\frac{2(\eta+1)}{\tau-c^0}\}$.
	We have:
	\begin{align*}
		\lambda_k^2 = \lambda_j^2 + \sum_{i=j}^{k-1}(\lambda_{i+1}^2 - \lambda_i^2) \leq \rbr{\frac{2(\eta+1)}{\tau-c^0}}^2 + \sum_{i=j}^{k-1}\rbr{2\lambda_i(\lambda_{i+1} - \lambda_i) + (\lambda_{i+1} - \lambda_i)^2}.
	\end{align*}
	Note that $\lambda_i>0$ for $j< i \leq k$.
	Therefore, with probability at least $1-6\delta$,
	\begin{align*}
		&\sum_{i=j}^{k-1}\rbr{2\lambda_i(\lambda_{i+1} - \lambda_i) + (\lambda_{i+1} - \lambda_i)^2} \leq \sum_{i=j}^{k-1}\rbr{2\lambda_i(\hatJ_i + \epsilon - \tau) + 1} \tag{definition of $\lambda_i$ and $|\lambda_{i+1}-\lambda_i|\leq 1$}\\
		&\leq \sum_{i=j}^{k-1}\rbr{2\lambda_i(2\E_i[\hatJ_i] + \epsilon - \tau) + 1 } + 32\lambda\ln\frac{4T}{\delta} \leq \sum_{i=j}^{k-1}\rbr{4\lambda_i(J^{\pi_i, c} + \epsilon - \tau) + 1} + 33\lambda\ln\frac{4T}{\delta} \tag{\pref{lem:e2r} with a union bound over $T$ possible values of $j$, $\lambda_i\leq\lambda$ by definition, and \pref{lem:hatJ}}\\
		&\leq 4\sum_{i=j}^{k-1}\rbr{\lambda_i(J^{\pi_i, c} + \epsilon - \tau) + 1} +  33\lambda\ln\frac{4T}{\delta} = 4\sum_{i=j}^{k-1}\rbr{\eta J^{\pi_i, r} - \eta J^{\pi_i, r - \frac{\lambda_i}{\eta}c} - \lambda_i(\tau-\epsilon) + 1} + 33\lambda\ln\frac{4T}{\delta}\\
		&\leq 4\sum_{i=j}^{k-1}\rbr{\eta J^{\pi_i, r} - \eta J^{\pi^0, r - \frac{\lambda_i}{\eta}c} - \lambda_i(\tau-\epsilon) + 1} + 33\lambda\ln\frac{4T}{\delta} + \frac{\eta\lambda}{\tau-c^0} \tag{\pref{lem:po}}\\ 
		%+ \tilO{\eta NH\rbr{1 + \frac{\lambda}{\eta}}\sqrt{S^3AK}}
		&\leq 4\sum_{i=j}^{k-1}\rbr{\eta - \frac{\tau-c^0}{2}\lambda_i + 1} + 33\lambda\ln\frac{4T}{\delta} + \frac{\eta\lambda}{\tau-c^0} \tag{$J^{\pi_i,r} - J^{\pi^0,r} \leq 1$, $J^{\pi^0,c}=c^0$, and $\epsilon\leq \frac{\tau-c^0}{2}$}\\ 
		&\leq 4(\eta+1) + 33\lambda\ln\frac{4T}{\delta} + \frac{\eta\lambda}{\tau-c^0}. \tag{$\lambda_i > \frac{2(\eta+1)}{\tau-c^0}$ for $j<i\leq k$}
		%= \tilO{k - j + \eta NH\rbr{1 + \frac{\lambda}{\eta}}\sqrt{S^3AK}}.
	\end{align*}
	Then by $\frac{\lambda}{4}>\frac{4(\eta+1)}{\tau-c^0}$ and $\eta\geq 132(\tau-c^0)\ln\frac{4T}{\delta}$, we have $\lambda_k = \lambda_k^2/\lambda_k \leq \frac{2(\eta+1)}{\tau-c^0} + 2 + \frac{\lambda}{8} + \frac{\lambda}{2}  < \lambda$.
	%+ \tilO{(\tau-c^0)k/\eta + NH\rbr{1 + \frac{\lambda}{\eta}}\sqrt{S^3AK}} < \lambda.$$
\end{proof}

%% file: app-weak.tex
% !TEX root = main.tex

%\paragraph{Notations} Define $P^k_h=P_{s^k_h, a^k_h}$.
%For a stationary policy $\pi\in(\Delta_{\calA})^{\calS}$, define $\tilpi$ as the policy that mimic $\pi$ in the finite-horizon MDP, such that $\tilpi(\cdot|s, h)=\pi(\cdot|s)$.
%Define $\optnu_s=\nu_{\tiloptpi, s}$.
%Define $\tiloptpi$ as the policy that mimic the $\optpi$ in the finite-horizon MDP, such that $\tiloptpi(\cdot|s, h)=\optpi(\cdot|s)$, and let $\optnu_s=\nu_{\tiloptpi, s}$.
%We also write $V^{\pi,P,d}_h$ as $V^{\pi,d}_h$ and $Q^{\pi,P,d}_h$ as $Q^{\pi,d}$ with true transition $P$ for convenience.

%\subsection{Transition Confidence Sets}
%\label{app:calP}
%We define the transition confidence set as follows:
%\begin{equation}
%	\label{eq:conf}
%	\calP_k = \cbr{ P'=\{P'_{s, a}\}_{(s, a)\in\SA}, P'_{s, a}\in\Delta_{\calS}: \abr{P'_{s, a}(s') - \P_{k, s, a}(s')} \leq 4\sqrt{\P_k(s'|s, a)\alpha_k(s, a) } + 28\alpha_k(s, a) },
%\end{equation}
%where $\alpha_k(s, a)=\frac{\iota'}{\Np_k(s, a)}$, $\iota'=\ln\frac{2SAT}{\delta}$, $\P_{k, s, a}=\frac{\N_k(s, a, s')}{\Np_k(s, a)}$, and $\N_k(s, a, s')$ is the number of visits to triplet $(s, a, s')$ before episode $k$.

As a standard practice, we first show that the true transition lies in the transition confidence sets with high probability, and provide some key lemmas related to transition estimation.
\begin{lemma}
	\label{lem:conf}
	With probability at least $1-\delta$, $\tilP\in\calP_k$,$\forall k$.
\end{lemma}
\begin{proof}
	For any $(s, a)\in\SA, s'\in\calS$, by \pref{lem:bernstein} and $N_{K+1}(s, a)\leq T$, we have with probability at least $1-\frac{\delta}{S^2A}$,
	\begin{align*}
		\abr{P_{s, a}(s') - \P_{k, s, a}(s')} \leq 4\sqrt{\P_{k, s, a}(s')\alpha_k(s, a)} + 28\alpha_k(s, a).
	\end{align*}
	By a union bound over all $(s, a)\in\SA$, $s'\in\calS$ and $\tilP_{s, a, h}=P_{s,a}$, the statement is proved.
\end{proof}

\begin{lemma}
	\label{lem:conf eps}
	Under the event of \pref{lem:conf}, $\abr{P'_{s, a, h}(s') - P_{s, a}(s')} \leq 8\sqrt{P_{s, a}(s')\alpha_k(s, a)} + 136\alpha_k(s, a) \triangleq \epsilon^{\star}_k(s, a, s')$ for any $P'\in\calP_k$.
\end{lemma}
\begin{proof}
	By $\tilP\in\calP_k$, we have for all $(s, a)\in\SA$, and $s'\in\calS$:
	\begin{align*}
		\P_{k, s, a}(s') \leq P_{s, a}(s') + 4\sqrt{\P_{k, s, a}(s')\alpha_k(s, a)} + 28\alpha_k(s, a).
	\end{align*}
	Applying $x^2\leq a x+b\implies x\leq a+\sqrt{b}$ with $a=4\sqrt{\alpha_k(s, a)}$ and $b=P_{s, a}(s')+28\alpha_k(s, a)$, we have
	$$\sqrt{\P_{k, s, a}(s')} \leq 4\sqrt{\alpha_k(s, a)}+\sqrt{P_{s, a}(s')+28\alpha_k(s, a)} \leq \sqrt{P_{s, a}(s')} + 10\sqrt{\alpha_k(s, a)}.$$
	Substituting this back to right-hand side of the inequality in \pref{eq:conf}, we have 
	$$4\sqrt{\P_{k,s,a}(s')\alpha_k(s, a) } + 28\alpha_k(s, a)\leq 4\sqrt{P_{s, a}(s')\alpha_k(s, a)} + 68\alpha_k(s, a).$$
	By $\tilP, P'\in\calP_k$, \pref{eq:conf}, and the triangle inequality $|P'_{s, a, h}(s') - P_{s, a}(s')|\leq |P'_{s, a, h}(s') - \P_{k,s,a}(s')|+|\P_{k,s,a}(s')-P_{s,a}(s')|$, the statement is proved.
\end{proof}

%%Also define $\optnu_s=\nu_{\tiloptpi, P, s}$.
\begin{lemma}
	\label{lem:opt}
	Under the event of \pref{lem:conf}, \pref{alg:weak} ensuers $\nu_{\tiloptpi,\tilP,s^k_1}$ lies in the domain of \pref{eq:OPT1}.
\end{lemma}
\begin{proof}
	By \pref{lem:VJ}, we have:
	\begin{align}
		\inner{\nu_{\tiloptpi,\tilP,s^k_1}}{c} &= V^{\tiloptpi, c}_1(s^k_1) \leq \sp^{\star}_c + HJ^{\optpi, c}\leq \sp^{\star}_c + H\tau.\label{eq:opt}
	\end{align}
	Then by $\tilP\in\calP_k$, the statement is proved.
\end{proof}

\begin{lemma}
	\label{lem:opt sp}
	Under the event of \pref{lem:conf}, \pref{alg:weak sp} ensures $\nu_{\tiloptpi,\tilP,s^k_1}$ lies in the domain of \pref{eq:OPT2}.
\end{lemma}
\begin{proof}
	Note that \pref{eq:opt} still holds.
	Moreover, by \pref{lem:VJ}, for any two states $s, s'$ and $h\in[H]$:
	\begin{align*}
		|V^{\tiloptpi}_h(s)-V^{\tiloptpi}_h(s')| &\leq |V^{\tiloptpi}_h(s) - (H-h+1)J^{\optpi,r}| + |V^{\tiloptpi}_h(s') - (H-h+1)J^{\optpi,r}| \leq 2\sp^{\star}_r,\\
		|V^{\tiloptpi,c}_h(s)-V^{\tiloptpi,c}_h(s')| &\leq |V^{\tiloptpi,c}_h(s) - (H-h+1)J^{\optpi,c}| + |V^{\tiloptpi,c}_h(s') - (H-h+1)J^{\optpi,c}| \leq 2\sp^{\star}_c.
	\end{align*}
	Then by $\tilP\in\calP_k$, the statement is proved.
\end{proof}

\subsection{\pfref{lem:VJ}}
\begin{proof}
	For any state $s$ and $h\in[H]$, we have:
	\begin{align*}
		V^{\tilpi, d}_h(s) - (H-h+1)J^{\pi,d} &= \E\sbr{\left. \sum_{h'=h}^H(d(s_{h'}, a_{h'}) - J^{\pi,d}) \right|\tilpi, \tilP, s_h=s}\\ 
		&= \E\sbr{\left. \sum_{h'=h}^H (q^{\pi, d}(s_{h'}, a_{h'}) - P_{s_{h'}, a_{h'}}v^{\pi, d}) \right|\tilpi, \tilP, s_h=s} \tag{\pref{eq:Bellman}}\\
		&= \E\sbr{\left. \sum_{h'=h}^H (v^{\pi, d}(s_{h'}) - v^{\pi, d}(s_{h'+1})) \right|\tilpi, \tilP, s_h=s}\tag{definition of $\tilpi$ and $\tilP$}\\
		&= v^{\pi, d}(s) - \E\sbr{\left.v^{\pi,d}(s_{H+1})\right|\tilpi, \tilP, s_h=s}.
	\end{align*}
	Thus, $|V^{\tilpi, d}_h(s) - (H-h+1)J^{\pi,d}|\leq \sp(v^{\pi,d})$ and the statement is proved.
\end{proof}

\subsection{\pfref{lem:V diff}}
\begin{proof}
	We condition on the event of \pref{lem:conf}, which happens with probability at least $1-\delta$.
	Note that with probability at least $1-\delta$:
	\begin{align*}
		\abr{\sumk (V^{\pi_k, P_k, d}_1(s^k_1) - V^{\pi_k, d}_1(s^k_1)) } &= \abr{\sumk\E\sbr{\left.\sumh (P_{k, s^k_h, a^k_h, h} - P^k_h)V^{\pi_k, P_k, d}_{h+1} \right| \pi_k, P}} \tag{\pref{lem:val diff H}}\\ 
		&\leq \sumk\E\sbr{\left.\sumh \abr{(P_{k, s^k_h, a^k_h, h} - P^k_h)V^{\pi_k, P_k, d}_{h+1}} \right| \pi_k, P} \tag{Jensen's inequality}\\
		&\leq 2\sumk\sumh \abr{(P_{k, s^k_h, a^k_h, h} - P^k_h)V^{\pi_k, P_k, d}_{h+1}} + \tilO{H^2}\tag{\pref{lem:e2r}}\\
		&= \tilO{\sqrt{S^2A\sumk\sumh\fV(P^k_h, V^{\pi_k, P_k, d}_{h+1})} + H^2S^2A } \tag{\pref{lem:dPV}}.
	\end{align*}
	Then by \pref{lem:sum var} and $H=(T/S^2A)^{1/3}$, with probability at least $1-2\delta$,
	\begin{align*}
		\abr{\sumk V^{\pi_k, d}_1(s^k_1) - V^{\pi_k,P_k,d}_1(s^k_1)} = \tilO{\sqrt{S^2A(H^2K + H^3S^2A)} + H^2S^2A} = \tilO{\sqrt{S^2AH^2K} + H^2S^2A}.
	\end{align*}
	This completes the proof.
\end{proof}

\subsection{\pfref{lem:V-d}}
\begin{proof}
	Define $\bar{V}^{\pi_k, P_k, d}_h(s) = V^{\pi_k, P_k, d}_h(s) - \min_{s'}V^{\pi_k, P_k, d}_h(s')$ and $\bar{Q}^{\pi_k, P_k, d}_h(s, a) = Q^{\pi_k, P_k, d}_h(s, a) - \min_{s'}V^{\pi_k, P_k, d}_h(s')$ so that $\bar{V}^{\pi_k, P_k, d}_h(s)\in[0, B]$ and
	\begin{align*}
		&\abr{\bar{Q}^{\pi_k, P_k, d}_h(s, a)} = \abr{Q^{\pi_k, P_k, d}_h(s, a) - V^{\pi_k, P_k, d}_h(\sstar)} \tag{$\sstar=\argmin_sV^{\pi_k, P_k, d}_h(s)$}\\
		&\leq \abr{d(s, a) - \suma\pi_k(a|\sstar)d(\sstar, a)} + \abr{P_{k, s, a}V^{\pi_k, P_k, d}_{h+1} - (P_k^{\pi_k})_{s,\cdot}V^{\pi_k,P_k,d}_{h+1}} \leq B + 1.
	\end{align*}
	Also define $\Ind_{s}(s')=\Ind\{s=s'\}$.
	Then with probability at least $1-2\delta$,
	\begin{align*}
		&\sumk\rbr{V^{\pi_k, P_k, d}_1(s^k_1) - \sumh d(s^k_h, a^k_h)}\\
		&= \sumk\rbr{ V^{\pi_k, P_k, d}_1(s^k_1) - Q^{\pi_k, P_k, d}_1(s^k_1, a^k_1) + Q^{\pi_k, P_k, d}_1(s^k_1, a^k_1) - d(s^k_1, a^k_1) - \sum_{h=2}^Hd(s^k_h, a^k_h) }\\
		&= \sumk\rbr{ V^{\pi_k, P_k, d}_1(s^k_1) - Q^{\pi_k, P_k, d}_1(s^k_1, a^k_1) + (P_{k, s^k_1, a^k_1, 1}-P^k_1)V^{\pi_k, P_k, d}_2 + (P^k_1-\Ind_{s^k_2})V^{\pi_k, P_k, d}_2 }\\
		&\qquad + \sumk\rbr{V^{\pi_k, P_k, d}_2(s^k_2) - \sum_{h=2}^Hd(s^k_h, a^k_h)}\\
		&= \sumk\sumh\rbr{  V^{\pi_k, P_k, d}_h(s^k_h) - Q^{\pi_k, P_k, d}_h(s^k_h, a^k_h) + (P_{k, s^k_h, a^k_h,h}-P^k_h)V^{\pi_k, P_k, d}_{h+1} + (P^k_h-\Ind_{s^k_{h+1}})V^{\pi_k, P_k, d}_{h+1}  }\tag{repeat the decomposition above}\\
		&= \sumk\sumh\rbr{  \bar{V}^{\pi_k, P_k, d}_h(s^k_h) - \bar{Q}^{\pi_k, P_k, d}_h(s^k_h, a^k_h) + (P_{k, s^k_h, a^k_h, h}-P^k_h)\bar{V}^{\pi_k, P_k, d}_{h+1} + (P^k_h-\Ind_{s^k_{h+1}})\bar{V}^{\pi_k, P_k, d}_{h+1}  }\\
		&= \tilO{(B+1)\sqrt{T} + BS\sqrt{AT} + BHS^2A}. \tag{\pref{lem:azuma}, \pref{lem:dPV}, and $\fV(P^k_h, \bar{V}_{h+1}^{\pi_k,P_k,d})\leq B^2$}
	\end{align*}
\end{proof}

\subsection{Auxiliary Lemmas}
\begin{lemma}
	\label{lem:sum var}
	Under the event of \pref{lem:conf}, for any utility function $d\in[0, 1]^{\SA}$, with probability at least $1-2\delta$, $\sumk\sumh\fV(P^k_h, V^{\pi_k,P_k,d}_{h+1}) = \tilo{ H^2(K+\sqrt{T}) + H^3S^2A }$.
\end{lemma}
\begin{proof}
	We decompose the variance into four terms:
	\begin{align*}
		&\sumk\sumh\fV(P^k_h, V^{\pi_k, P_k, d}_{h+1}) = \sumk\sumh \rbr{P^k_h(V^{\pi_k, P_k, d}_{h+1})^2 - (P^k_hV^{\pi_k, P_k, d}_{h+1})^2}\\
		&= \sumk\sumh \rbr{P^k_h(V^{\pi_k, P_k, d}_{h+1})^2 - V^{\pi_k, P_k, d}_{h+1}(s^k_{h+1})^2} + \sumk\sumh \rbr{V^{\pi_k, P_k, d}_{h+1}(s^k_{h+1})^2 - V^{\pi_k, P_k, d}_h(s^k_h)^2}\\
		&\qquad + \sumk\sumh \rbr{V^{\pi_k, P_k, d}_h(s^k_h)^2 - Q^{\pi_k, P_k, d}_h(s^k_h, a^k_h)^2} + \sumk\sumh \rbr{Q^{\pi_k, P_k, d}_h(s^k_h, a^k_h)^2 - (P^k_hV^{\pi_k, P_k, d}_{h+1})^2}.
	\end{align*}
	For the first term, by \pref{lem:freedman}, with probability at least $1-\delta$,
	\begin{align*}
		\sumk\sumh P^k_h(V^{\pi_k, P_k, d}_{h+1})^2 - V^{\pi_k, P_k, d}_{h+1}(s^k_{h+1})^2 &= \tilO{\sqrt{\sumk\sumh \fV(P^k_h, (V^{\pi_k, P_k, d}_{h+1})^2)} + H^2}\\ 
		&= \tilO{H\sqrt{\sumk\sumh \fV(P^k_h, V^{\pi_k, P_k, d}_{h+1})} + H^2}. \tag{\pref{lem:var XY}}
	\end{align*}
	The second term is upper bounded by $0$ by $V^{\pi_k,P_k,d}_{H+1}(s)=0$ for $s\in\calS$.
	For the third term, by Cauchy-Schwarz inequality and \pref{lem:azuma}, with probability at least $1-\delta$:
	\begin{align*}
		\sumk\sumh V^{\pi_k, P_k, d}_h(s^k_h)^2 - Q^{\pi_k, P_k, d}_h(s^k_h, a^k_h)^2 &\leq \sumk\sumh \rbr{\sum_a\pi_k(a|s^k_h, h)Q^{\pi_k, P_k, d}_h(s^k_h, a)^2 -  Q^{\pi_k, P_k, d}_h(s^k_h, a^k_h)^2}\\
		&= \tilO{H^2\sqrt{T}}.
	\end{align*}
	For the fourth term, by $a^2-b^2=(a+b)(a-b)$ and $\norm{V^{\pi_k,P_k,d}_h}_{\infty},\norm{Q^{\pi_k,P_k,d}_h}_{\infty}\leq H$:
	\begin{align*}
		&\sumk\sumh Q^{\pi_k, P_k, d}_h(s^k_h, a^k_h)^2 - (P^k_hV^{\pi_k, P_k, d}_{h+1})^2 \leq 2H\sumk\sumh\abr{Q^{\pi_k, P_k, d}_h(s^k_h, a^k_h) - P^k_hV^{\pi_k, P_k, d}_{h+1}}\\
		&\leq 2H^2K + 2H\sumk\sumh\abr{(P_{k, s^k_h, a^k_h, h} - P^k_h)V^{\pi_k,P_k,d}_{h+1}} = \tilO{H^2K + H\sqrt{S^2A\sumk\sumh\fV(P^k_h, V^{\pi_k, P_k, d}_{h+1})} + H^3S^2A }. \tag{\pref{lem:dPV}}
	\end{align*}
	Putting everything together, we have
	\begin{align*}
		&\sumk\sumh\fV(P^k_h, V^{\pi_k, P_k, d}_{h+1})\\
		&= \tilO{ H\sqrt{\sumk\sumh \fV(P^k_h, V^{\pi_k, P_k, d}_{h+1})} + H^2(K+\sqrt{T}) + H\sqrt{S^2A\sumk\sumh\fV(P^k_h, V^{\pi_k, P_k, d}_{h+1}) } + H^3S^2A }.
	\end{align*}
	Solving a quadratic inequality, we get $\sumk\sumh\fV(P^k_h, V^{\pi_k, P_k, d}_{h+1}) = \tilo{ H^2(K+\sqrt{T}) + H^3S^2A }$.
\end{proof}

\begin{lemma}
	\label{lem:dPV}
	Under the event of \pref{lem:conf}, for any value function $V$ with $V_h\in[0, B]^{\calS},\forall h\in[H]$, we have:
	\begin{align*}
		\sumk\sumh\abr{(P_{k,s^k_h,a^k_h, h}-P^k_h)V_{h+1}}&=\tilO{\sqrt{S^2A\sumk\sumh\fV(P^k_h, V_{h+1})} + BHS^2A }.
		%\sumk\min\cbr{B, \sumh\abr{(P_{k,s^k_h,a^k_h, h}-P^k_h)V_{h+1}} }&=\tilO{\sqrt{S^2A\sumk\sumh\fV(P^k_h, V_{h+1})} + BS^2A }.
	\end{align*}
\end{lemma}
\begin{proof}
	Define $\Ind_k=\Ind\{\forall (s, a): N_{k+1}(s, a)\leq 2N_k(s, a) \}$ and $z^k_h(s')=V_{h+1}(s') - P^k_hV_{h+1}$.
	By \pref{lem:conf eps},
	%For the first term, by \pref{lem:conf eps},
	\begin{align*}
		&\sumk\sumh\abr{(P_{k,s^k_h,a^k_h,h}-P^k_h)V_{h+1}} = \sumk\sumh\abr{(P_{k,s^k_h,a^k_h,h}-P^k_h)z^k_h} \leq \sumk\sumh\min\cbr{B, \sum_{s'}\epsilon^{\star}_k(s^k_h, a^k_h, s')|z^k_h(s')|}\\
		&\leq 2\sumk\sumh\min\cbr{B, \sum_{s'}\epsilon^{\star}_{k+1}(s^k_h, a^k_h, s')|z^k_h(s')|} + BH\sumk\Ind_k^c.
	\end{align*}
%	For the second term,
%	\begin{align*}
%		&\sumk\min\cbr{B, \sumh\abr{(P_{k,s^k_h,a^k_h,h}-P^k_h)V_{h+1}}} = \sumk\min\cbr{B, \sumh\abr{(P_{k,s^k_h,a^k_h,h}-P^k_h)z^k_h} }\\ 
%		&\leq \sumk\min\cbr{B, \sumh\sum_{s'}\epsilon^{\star}_k(s^k_h, a^k_h, s')|z^k_h(s')|} \leq 2\sumk\min\cbr{B, \sumh\sum_{s'}\epsilon^{\star}_{k+1}(s^k_h, a^k_h, s')|z^k_h(s')|} + B\sumk\Ind_k^c.
%	\end{align*}
	Note that $\sumk\Ind_k^c=\tilo{SA}$ by definition.
	Thus it suffices to bound $\sumk\sumh\sum_{s'}\epsilon^{\star}_{k+1}(s^k_h, a^k_h, s')|z^k_h(s')|$.
	Note that:
	\begin{align*}
		&\sumk\sumh\sum_{s'}\epsilon^{\star}_{k+1}(s^k_h, a^k_h, s')|z^k_h(s')| = \tilO{\sumk\sumh\sum_{s'}\sqrt{\frac{P^k_h(s')z^k_h(s')^2}{\Np_{k+1}(s^k_h, a^k_h)}} + \sumk\sumh\frac{SB}{\Np_{k+1}(s^k_h, a^k_h)}} \tag{definition of $\epsilon_k^{\star}$}\\
		&= \tilO{ \sumk\sumh\sqrt{\frac{S\fV(P^k_h, V_{h+1})}{\Np_{k+1}(s^k_h, a^k_h)}} + BS^2A }\\
		&= \tilO{ \sqrt{\sumk\sumh\frac{S}{\Np_{k+1}(s^k_h, a^k_h)}}\sqrt{\sumk\sumh\fV(P^k_h, V_{h+1})} + BS^2A } \tag{Cauchy-Schwarz inequality}\\
		&= \tilO{\sqrt{S^2A\sumk\sumh\fV(P^k_h, V_{h+1})} + BS^2A }.
	\end{align*}
	Plugging these back completes the proof.
\end{proof}

\begin{lemma}\citep[Lemma 1]{efroni2020optimistic}
	\label{lem:val diff H}
	For any policy $\pi\in(\Delta_{\calA})^{\calS\times[H]}$, two transition functions $P, P'$, and utility function $d\in\fR^{\SA}$, we have $V^{\pi,P,d}_1(s) - V^{\pi,P',d}_1(s)=\E[\sumh (P_{s_h, a_h, h}-P'_{s_h, a_h, h})V^{\pi,P,d}_{h+1}|\pi,P',s_1=s]$.
\end{lemma}

\begin{lemma}{\citep[Lemma 30]{chen2021implicit}}
	\label{lem:var XY}
	For a random variable $X$ such that $|X|\leq C$, we have: $\var[X^2]\leq 4C^2\var[X]$.
\end{lemma}

%% file: app-aux.tex
% !TEX root = main.tex

\begin{lemma}[Any interval Azuma's inequality]
	\label{lem:azuma}
	Let $\{X_i\}_{i=1}^{\infty}$ be a martingale difference sequence and $|X_i|\leq B$ almost surely.
	Then with probability at least $1-\delta$, for any $l, n$: $\abr{\sum_{i=l}^{l+n-1}X_i}\leq B\sqrt{2n\ln\frac{4(l+n-1)^3}{\delta}}$.
\end{lemma}
\begin{proof}
	For each $l, n\in\fN_+$, we have with probability at least $1-\frac{\delta}{2(l+n-1)^3}$, $\abr{\sum_{i=l}^{l+n-1}X_i}\leq B\sqrt{2n\ln\frac{4(l+n-1)^3}{\delta}}$ by \citep[Lemma 20]{chen2021finding}.
	The statement is then proved by a union bound (note that $\sum_{l=1}^{\infty}\sum_{n=1}^{\infty}\frac{1}{2(l+n-1)^3}=\sum_{i=1}^{\infty}\sum_{j=1}^i\frac{1}{2i^3}=\sum_{i=1}^{\infty}\frac{1}{2i^2}\leq 1$).
\end{proof}

\begin{lemma}\citep[Lemma 38]{chen2021improved}
	\label{lem:freedman}
	Let $\{X_i\}_{i=1}^{\infty}$ be a martingale difference sequence adapted to the filtration $\{\calF_i\}_{i=0}^{\infty}$ and $|X_i|\leq B$ for some $B>0$.
	Then with probability at least $1-\delta$, for all $n\geq 1$ simultaneously,
	\begin{align*}
		\abr{\sum_{i=1}^nX_i}\leq 3\sqrt{\sum_{i=1}^n\E[X_i^2|\calF_{i-1}]\ln\frac{4B^2n^3}{\delta} } + 2B\ln\frac{4B^2n^3}{\delta}.
	\end{align*}
\end{lemma}

\begin{lemma}\citep{weissman2003inequalities}
	\label{lem:weiss}
	Given a distribution $p\in\Delta_m$ and let $\bar{p}$ be an empirical distribution of $p$ over $n$ samples.
	Then, $\norm{p-\bar{p}}_1\leq\sqrt{m\ln\frac{2}{\delta}/n}$ with probability at least $1-\delta$.
\end{lemma}

\begin{lemma}\citep[Theorem D.3]{cohen2020near}
	\label{lem:bernstein}
	Let $\{X_n\}_{n=1}^{\infty}$ be a sequence of i.i.d random variables with expectation $\mu$ and $X_n\in[0, B]$ almost surely.
	Then with probability at least $1-\delta$, for any $n\geq 1$:
	\begin{align*}
		\abr{\sum_{i=1}^n(X_i-\mu)} \leq \min\cbr{2\sqrt{B\mu n\ln\frac{2n}{\delta}} + B\ln\frac{2n}{\delta}, 2\sqrt{B\sum_{i=1}^nX_i\ln\frac{2n}{\delta}} + 7B\ln\frac{2n}{\delta}}.
	\end{align*}
\end{lemma}

\begin{lemma}{\citep[Lemma D.4]{cohen2020near} and \citep[Lemma E.2]{cohen2021minimax}}
	\label{lem:e2r}
	Let $\{X_i\}_{i=1}^{\infty}$ be a sequence of random variables w.r.t to the filtration $\{\calF_i\}_{i=0}^{\infty}$ and $X_i\in[0,B]$ almost surely.
	Then with probability at least $1-\delta$, for all $n\geq 1$ simultaneously:
	\begin{align*}
		\sum_{i=1}^n\E[X_i|\calF_{i-1}] &\leq 2\sum_{i=1}^n X_i + 4B\ln\frac{4n}{\delta},\\
		\sum_{i=1}^n X_i &\leq 2\sum_{i=1}^n\E[X_i|\calF_{i-1}] + 8B\ln\frac{4n}{\delta}.
	\end{align*}
\end{lemma}